%% file: main.tex
\documentclass{article}
\pdfoutput=1

\usepackage{fullpage}

\usepackage{natbib}
\usepackage[utf8]{inputenc} 
\usepackage[T1]{fontenc}    
\usepackage{hyperref}       
\usepackage{url}            
\usepackage{booktabs}       
\usepackage{amsfonts}       
\usepackage{nicefrac}       
\usepackage{microtype}      

\usepackage{caption}    
\usepackage{subcaption}    
\usepackage{tabularx}
\usepackage{graphicx}
\usepackage{hyperref}
\usepackage{color}
\usepackage{amssymb}
\usepackage{amsmath}
\usepackage{amsthm}
\usepackage{mathrsfs}
\usepackage{tikz}
\usepackage{enumerate}
\usepackage{algorithm}
\usepackage{algorithmic}
\usepackage{multirow}
\usepackage{bbm}
\usepackage{thmtools,thm-restate}
\usepackage{xcolor}

\usepackage{blkarray, bigstrut}
\usepackage{verbatim}

\declaretheorem{theorem}

\newcommand\xx{\mathbf{x}}

\newcommand\RR{\mathbb{R}}

\newcommand\hh{\mathbf{h}}
\newcommand\ww{\mathbf{w}}

\newcommand{\E}{{\mathbb{E}}}
\newcommand{\R}{{\mathbb{R}}}
\newcommand{\eps}{{\epsilon}}
\newcommand{\hx}{{\hat{x}}}

\newcommand{\red}[1]{\textcolor{red}{#1}}

\newcommand{\ns}[1]{\| #1 \|^2}
\newcommand{\n}[1]{\| #1 \|}

\newcommand{\tdo}{\widetilde{O}}
\newcommand{\tdomega}{\widetilde{\Omega}}
\newcommand\icu{\mu_{i\rightarrow}}

\input{macros.tex}

        {\hspace*{\fill}$\Box$\par\vspace{4mm}}
\newenvironment{proofof}[1]{\smallskip\noindent{\bf Proof of #1.}}%
        {\hspace*{\fill}$\Box$\par}

\DeclareMathOperator{\spn}{span}

\newtheorem{lemma}{Lemma}

\newtheorem{definition}{Definition}

\newtheorem{remark}{Remark}

\hypersetup{
    colorlinks=true,       
    linkcolor=blue,    
    citecolor=blue,        
    filecolor=magenta,     
    urlcolor=blue
}

\newcommand{\rcknote}[1]{\noindent{\red{RCK: #1 }}}

\title{Explaining Landscape Connectivity of Low-cost Solutions for Multilayer Nets}

\author{Rohith Kuditipudi \\ Duke University \\ \texttt{rohith.kuditipudi@duke.edu}  \and Xiang Wang \\ Duke University \\ \texttt{xwang@cs.duke.edu} \and Holden Lee \\ Princeton University \\ \texttt{holdenl@princeton.edu} \and Yi Zhang \\ Princeton University \\ \texttt{y.zhang@cs.princeton.edu} \and Zhiyuan Li \\ Princeton University \\ \texttt{zhiyuanli@cs.princeton.edu} \and Wei Hu \\ Princeton University \\ \texttt{huwei@cs.princeton.edu} \and Sanjeev Arora \\ Princeton University and Institute for Advanced Study \\ \texttt{arora@cs.princeton.edu} \and Rong Ge \\ Duke University \\ \texttt{rongge@cs.duke.edu}}

\date{}

\begin{document}

\maketitle

\begin{abstract}
{\em Mode connectivity}~\citep{garipov2018loss,draxler2018essentially} is a surprising phenomenon in the loss landscape of deep nets. Optima---at least those discovered by gradient-based optimization---turn out to be connected by simple paths on which the loss function is almost constant. Often, these paths can be chosen to be piece-wise linear, with as few as two segments. 

We give mathematical explanations for this phenomenon, assuming generic properties (such as dropout stability and noise stability) of well-trained deep nets, which have previously been identified as part of understanding the generalization properties of deep nets. Our explanation holds for realistic multilayer nets, and experiments are presented to verify the theory. 
\end{abstract}
\input{intro}
\input{prelim}
\input{dropout}
\input{noisestable}
\input{examples_main}
\input{experiments}

\subsubsection*{Acknowledgments} Rong Ge acknowledges funding from NSF CCF-1704656, NSF CCF-1845171 (CAREER), the Sloan Fellowship and Google Faculty Research Award. Sanjeev Arora acknowledges funding from the NSF, ONR, Simons Foundation, Schmidt Foundation, Amazon Research, DARPA and SRC.

\bibliography{bib}
\bibliographystyle{apalike}

\newpage
\appendix

\input{dropout_proofs}
\input{noisestability_proof}

\input{counterexample.tex}

\input{appendix_experiments}

\input{tools}

\end{document}

%% file: macros.tex
\newcommand{\nequiv}[0]{\not\equiv}

\newcommand{\ba}[1]{\left[ {#1} \right]}

\newcommand{\fl}[1]{\left\lfloor {#1}\right\rfloor}
\newcommand{\pa}[1]{\left( {#1} \right)}

\newcommand{\ve}[1]{\left\Vert {#1}\right\Vert}

\newcommand{\rc}[1]{\frac{1}{#1}}

\newcommand{\fc}[2]{\frac{#1}{#2}}

\newcommand{\pf}[2]{\pa{\frac{#1}{#2}}}

\newcommand{\mattg}[4]{
\left[\begin{array}{c|c}
#1 & #2 \\
\hline
#3 & #4
\end{array}\right]
}
\newcommand{\ctwog}[2]{
\left[\begin{array}{c}
#1 \\
\hline
#2
\end{array}\right]
}
\newcommand{\rtwog}[2]{
\left[\begin{array}{c|c}
#1 & #2 \\
\end{array}\right]
}

\newcommand{\prodo}[2]{\prod_{#1=1}^{#2}}

\newcommand{\vocab}[1]{\textbf{#1}}

%% file: intro.tex
\section{Introduction}
Efforts to understand how and why deep learning works have led to a focus on the {\em optimization landscape} of the training loss. Since optimization to near-zero training loss occurs for many choices of random initialization,  it is clear that the landscape contains many global optima (or near-optima). However, the loss can become quite high when interpolating between found optima, suggesting that these optima occur at the bottom of \textquotedblleft valleys\textquotedblright\ surrounded on all sides by high walls.  Therefore the phenomenon of {\em mode connectivity}~\citep{garipov2018loss,draxler2018essentially} came as a surprise: 
optima (at least the ones discovered by gradient-based optimization) are connected by simple paths in the parameter space, on which the loss function is almost constant. In other words, the optima are not walled off in separate valleys as hitherto believed. More surprisingly,  the paths connecting discovered optima can be piece-wise linear with as few as two segments. 

Mode connectivity begs for theoretical explanation. One paper~\citep{freeman2016topology} attempted such an explanation for $2$-layer nets, even before the discovery of the phenomenon in multilayer nets. However, they require the width of the net to be exponential in some relevant parameters. Others \citep{venturi2018spurious, liang2018understanding, nguyen2018loss, nguyen2019connected} require special structure in their networks where the number of neurons needs to be greater than the number of training data points. Thus it remains an open problem to explain mode connectivity even in the 2-layer case with realistic parameter settings, let alone for standard multilayer architectures.

At first sight, finding a mathematical explanation of the mode connectivity phenomenon for multilayer nets---e.g., for a $50$-layer ResNet on ImageNet---appears very challenging. However, the glimmer of hope is that since the phenomenon exists for a variety of architectures and datasets, it must arise from some generic property of trained nets. The fact that the connecting paths between optima can have as few as two linear segments further bolsters this hope.

Strictly speaking, empirical findings such as in~\citep{garipov2018loss,draxler2018essentially} do not show connectivity 
between {\em all} optima, but only for {\em typical} optima discovered by gradient-based optimization.  It seems an open question whether  connectivity holds for all optima in overparametrized nets. 
Section~\ref{sec:example:main} answers this question,  via a simple example of an
overparametrized two-layer net, not all of whose optima are connected via low-cost paths.

Thus to explain mode connectivity one must seek generic properties that hold for optima obtained via gradient-based optimization on realistic data. A body of work that could be a potential source of such generic properties is the ongoing effort to understand the generalization puzzle of over-parametrized nets---specifically, to understand the \textquotedblleft true model capacity\textquotedblright. For example, \citet{morcos2018importance} note that networks that generalize are insensitive to linear restrictions in the parameter space. \citet{arora2018stronger} define a {\em noise stability} property of deep nets, whereby adding Gaussian noise to the output of a layer is found to have minimal effect on the vector computed at subsequent layers. Such properties seem to arise in a variety of architectures purely from gradient-based optimization, without any explicit noise-injection during training---though of course using small-batch gradient estimates is an implicit source of noise-injection. (Sometimes training also explicitly injects noise, e.g. dropout or batch-normalization, but that is not needed for noise stability to emerge.)

Since resilience to perturbations arises in a variety of architectures, such resilience counts as a \textquotedblleft generic\textquotedblright\  property for which
it is natural to prove mode connectivity as a consequence. We carry this out in the current paper. 
Note that our goal here is not to explain every known detail of mode connectivity, but rather to give a plausible first-cut explanation.

First, in Section~\ref{sec:dropout} we explain mode connectivity by assuming the network is trained via dropout. In fact, the desired property is weaker: so long as there {\em exists} even a {\em single} dropout pattern that keeps the training loss close to optimal on the two solutions, our proof constructs a piece-wise linear path between them. 
The number of linear segments grows linearly with the depth of the net.  

Then, in Section~\ref{sec:noise_stable} we make a stronger assumption of noise stability along the lines of \cite{arora2018stronger} and show that it implies mode connectivity using paths with $10$ linear segments. While this assumption is strong, it appears to be close to what is satisfied in practice. (Of course, one could explicitly train deep nets to satisfy the needed noise stability assumption, and the theory applies directly to them.)

\subsection{Related work}
The landscape of the loss function for training neural networks has received a lot of attention. \citet{dauphin2014identifying, choromanska2015loss} conjectured that local minima of multi-layer neural networks have similar loss function values, and proved the result in idealized settings. For linear networks, it is known \citep{kawaguchi2016deep} that all local minima are also globally optimal. 

Several theoretical works have explored whether a neural network has spurious valleys (non-global minima that are surrounded by other points with higher loss). \citet{freeman2016topology} showed that for a two-layer net, if it is sufficiently overparametrized then all the local minimizers are (approximately) connected. However, in order to guarantee a small loss along the path they need the number of neurons to be exponential in the number of input dimensions.
\citet{venturi2018spurious} proved that if the number of neurons is larger than either the number of training samples or the intrinsic dimension (infinite for standard architectures), then the neural network cannot have spurious valleys. \citet{liang2018understanding} proved similar results for the binary classification setting. \citet{nguyen2018loss, nguyen2019connected} relaxed the requirement on overparametrization, but still require the output layer to have more direct connections than the number of training samples. 

Some other papers have studied the existence of spurious local minima. \citet{yun2018critical} showed that in most cases neural networks have spurious local minima. Note that a local minimum need only have loss no larger than the points in its neighborhood, so a local minimum is not necessarily a spurious valley. \citet{safran2018spurious} found spurious local minima for simple two-layer neural networks under a Gaussian input distribution. These spurious local minima are indeed spurious valleys as they have positive definite Hessian. 



%% file: prelim.tex
\section{Preliminaries}


\paragraph{Notations} For a vector $v$, we use $\|v\|$ to denote its $\ell_2$ norm. For a matrix $A$, we use $\|A\|$ to denote its operator norm, and $\|A\|_F$ to denote its Frobenius norm. We use $[n]$ to denote the set $\{1,2,\ldots,n\}$. We use $I_n$ to denote the identity matrix in $\R^{n\times n}$. We use $O(\cdot), \Omega(\cdot)$ to hide constants and use $\tdo(\cdot),\tdomega(\cdot)$ to hide poly-logarithmic factors.

\paragraph{Neural network} In most of the paper, we consider fully connected neural networks with ReLU activations. Note however that our results can also be extended to convolutional neural networks (in particular, see Remark~\ref{rmk:channelwise} and the experiments in Section~\ref{sec:experiments}). 

Suppose the network has $d$ layers. Let the vector before activation at layer $i$ be $x^i$, $i\in[d],$ where $x^d$ is just the output. For convenience, we also denote the input $x$ as $x^0$. Let $A_i$ be the weight matrix at $i$-th layer, so that we have $x^{i} = A_{i} \phi(x^{i-1})$ for $2\leq i\leq d$ and $x^1=A_1x^0$. For any layer $i,\ 1\leq i\leq d,$ let the width of the layer be $h_i$. We use $[A_i]_j$ to denote the $j$-th column of $A_i.$ Let the maximum width of the hidden layers be $h_{\max}:=\max\{h_1, h_2, \ldots, h_{d-1}\}$ and the minimum width of the hidden layers be $h_{\min}:=\min\{h_1, h_2, \ldots, h_{d-1}\}$. 

We use $\Theta$ to denote the set of parameters of neural network, and in our specific model, $\Theta =  \R^{h_1\times h_0} \times\R^{h_2\times h_1}\times  \cdots \times \R^{h_d\times h_{d-1}}$ which consists of all the weight matrices $\{A_i\}$'s. 

Throughout the paper, we use $f_\theta$, $\theta\in \Theta$ to denote the function that is computed by the neural network. For a data set $(x,y)\sim \mathcal{D}$, the loss is defined as $L_\mathcal{D}(f_\theta): = \E_{(x,y)\sim \mathcal{D}}[l(y,f_\theta(x))]$ where $l$ is a loss function. The loss function $l(y, \hat{y})$ is convex in the second parameter. We omit the distribution $\mathcal{D}$ when it is clear from the context.

\paragraph{Mode connectivity and spurious valleys} Fixing a neural network architecture, a data set $\mathcal{D}$ and a loss function, we say two sets of parameters/solutions $\theta^A$ and $\theta^B$ are \vocab{$\epsilon$-connected} if there is a path $\pi(t): \R \to \Theta$ that is continuous with respect to $t$ and satisfies: 1. $\pi(0) = \theta^A$; 2.  $\pi(1) = \theta^B$ and 3. for any $t\in [0,1]$, $L(f_{\pi(t)}) \le \max\{L(f_{\theta^A}), L(f_{\theta^B})\}+\epsilon$. If $\epsilon = 0$, we omit $\epsilon$ and just say they are connected.

If all local minimizers are connected, then we say that the loss function has the \vocab{mode connectivity property}. However, as we later show in Section~\ref{sec:example:main}, this property is very strong and is not true even for overparametrized two-layer nets. Therefore we restrict our attention to classes of low-cost solutions that can be found by the gradient-based algorithms (in particular in Section~\ref{sec:dropout} we focus on solutions that are dropout stable, and in Section~\ref{sec:noise_stable} we focus on solutions that are noise stable). We say the loss function has $\epsilon$-\vocab{mode connectivity property} with respect to a class of low-cost solutions $\mathcal{C}$, if any two minimizers in $\mathcal{C}$ are $\epsilon$-connected.




Mode connectivity is closely related to the notion of spurious valleys and connected sublevel sets \citep{venturi2018spurious}. 
If a loss function has all its sublevel sets ($\{\theta:L(f_\theta)\le \lambda\}$) connected, then it has the mode connectivity property.
When the network only has the mode connectivity property with respect to a class of solutions $\mathcal{C}$, as long as the class $\mathcal{C}$ contains a global minimizer, we know there are no spurious valleys in $\mathcal{C}$. 


However, we emphasize that neither mode connectivity or lack of spurious valleys implies any local search algorithm can efficiently find the global minimizer. These notions only suggest that it is unlikely for local search algorithms to get completely stuck.

%% file: dropout.tex
\definecolor{forestgreen}{rgb}{0.0, 0.5, 0.0}

\section{Connectivity of dropout-stable optima}
\label{sec:dropout}
In this section we show that \emph{dropout stable} solutions are connected. More concretely, we define a solution $\theta$ to be $\epsilon$-dropout stable if we can remove a subset of half its neurons in each layer such that the loss remains steady.



\begin{definition} \label{def:drop-stability}
(Dropout Stability) A solution $\theta$ is \vocab{$\epsilon$-dropout stable} if for all $i$ such that $1 \leq i < d$, 
there exists a subset of at most $\lfloor h_j/2 \rfloor$ hidden units
in each of the layers $j$ from $i$ through $d-1$ such that after rescaling the outputs of these hidden units (or equivalently, the corresponding rows and/or columns of the relevant weight matrices) by some factor $r$\footnote{Note our results will also work if $r$ is allowed to vary for each layer.} and setting the outputs of the remaining units to zero, we obtain a parameter $\theta_i$ such that $L(f_{\theta_i}) \leq L(f_\theta) + \epsilon$.
\end{definition}

Intuitively, if a solution is $\epsilon$-dropout stable then it is essentially only using half of the network's capacity. We show that such solutions are connected:

\begin{theorem}
\label{thm:connect_via_dropout}
Let $\theta^A$ and $\theta^B$ be two $\epsilon$-dropout stable solutions. Then there exists a path in parameter space $\pi : [0,1] \to \Theta$ between $\theta^A$ and $\theta^B$ such that $L(f_{\pi(t)}) \leq \max \{L(f_{\theta^A}),L(f_{\theta^B})\}+\epsilon$ for $0 \leq t \leq 1$.
In other words, letting $\mathcal C$ be the set of solutions that are $\epsilon$-dropout stable, a ReLU network has the $\epsilon$-mode connectivity property with respect to $\mathcal C$.
\end{theorem}

Our path construction in Theorem~\ref{thm:connect_via_dropout} consists of two key steps. First we show that we can rescale at least half the hidden units in both $\theta^A$ and $\theta^B$ to zero via continuous paths of low loss, thus obtaining two parameters $\theta_1^A$ and $\theta_1^B$ satisfying the criteria in Definition~\ref{def:drop-stability}. 

\begin{lemma}\label{lem:original_sparse}
Let $\theta$ be an $\epsilon$-dropout stable solution and let $\theta_i$ be specified as in Definition \ref{def:drop-stability} for $1 \leq i < d$. Then there exists a path in parameter space $\pi: [0,1] \to \Theta$ between $\theta$ and $\theta_1$ passing through each $\theta_i$ such that $L(f_{\pi(t)}) \leq L(f_\theta) + \epsilon$ for $0 \leq t \leq 1$.
\end{lemma}

Though na\"ively one might expect to be able to directly connect the weights of $\theta$ and $\theta_1$ via interpolation, such a path may incur high loss as the loss function is not convex over $\Theta$. In our proof of Lemma~\ref{lem:original_sparse}, we rely on a much more careful construction. The construction uses two types of steps: (a) interpolate between two weights in the top layer (the loss is convex in the top layer weights); (b) if a set of neurons already have their output weights set to zero, then we can change their input weights arbitrarily. See Figure~\ref{fig:path-ex} for an example path for a 3-layer network. Here we have separated the weight matrices into equally sized blocks: $A_3 = \rtwog{L_3}{R_3}$, $A_2 = \mattg{L_2}{C_2}{D_2}{R_2}$ and $A_1 = \ctwog{L_1}{B_1}$. The path consists of 6 steps alternating between type (a) and type (b). Note that for all the type (a) steps, we only update the top layer weights; for all the type (b) steps, we only change rows of a weight matrix (inputs to neurons) if the corresponding columns in the previous matrix (outputs of neurons) are already 0. In Section~\ref{sec:dropout_proof} we show how such a path can be generalized to any number of layers.
\begin{figure}[ht]
    \centering
    \begin{align*}
    \begin{matrix}
    & A_3 & A_2 & A_1 & \\
    (1) & \rtwog{L_3}{R_3}& \mattg{L_2}{C_2}{D_2}{R_2} & \ctwog{L_1}{B_1} & \\[1em]
    (2) & \rtwog{\color{red}rL_3}{\color{red}0}& \mattg{L_2}{C_2}{D_2}{R_2} & \ctwog{L_1}{B_1} & (a)\\[1em]
    (3) & \rtwog{rL_3}{\color{forestgreen}0}& \mattg{L_2}{C_2}{\color{red}rL_2}{\color{red}0} & \ctwog{L_1}{B_1} & (b)\\[1em]
    (4) & \rtwog{\color{red}0}{\color{red}rL_3}& \mattg{L_2}{C_2}{rL_2}{0} & \ctwog{L_1}{B_1} & (a)\\[1em]
    (5) & \rtwog{\color{forestgreen}0}{rL_3}& \mattg{\color{red}rL_2}{\color{red}0}{rL_2}{0} & \ctwog{L_1}{B_1} & (b)\\[1em]
    (6) & \rtwog{\color{red}rL_3}{\color{red}0}& \mattg{rL_2}{0}{rL_2}{0} & \ctwog{L_1}{B_1} & (a)\\[1em]
    (7) & \rtwog{rL_3}{\color{forestgreen}0}& \mattg{rL_2}{\color{forestgreen}0}{\color{red}0}{\color{forestgreen}0} & \ctwog{L_1}{\color{red}0} & (b)
    \end{matrix}
    \end{align*}
    
    \caption{Example path, 6 line segments from a 3-layer network to its dropout version. Red denotes weights that have changed between steps while green denotes the zeroed weights that allow us to make these changes without affecting our output.}
    \label{fig:path-ex}
\end{figure}

We then show that we can permute the hidden units of $\theta_1^A$ such that its non-zero units do not intersect with those of $\theta_1^B$, thus allowing us two interpolate between these two parameters. This is formalized in the following lemma and the proof is deferred to supplementary material.

\begin{lemma}\label{lem:drop-connect}
Let $\theta$ and $\theta'$ be two solutions such that at least $\lceil h_i/2 \rceil$ of the units in the $i^{th}$ hidden layer have been set to zero in both. Then there exists a path in parameter space $\pi: [0,1] \to \Theta$ between $\theta$ and $\theta'$ with 8 line segments such that $L(f_{\pi(t)}) \leq \max\{L(f_\theta),L(f_{\theta'})\}$.
\end{lemma}

Theorem~\ref{thm:connect_via_dropout} follows immediately from Lemma~\ref{lem:original_sparse} and Lemma~\ref{lem:drop-connect}, as one can first connect $\theta^A$ to its dropout version $\theta^A_1$ using Lemma~\ref{lem:original_sparse}, then connect $\theta^A_1$ to dropout version $\theta^B_1$ of $\theta^B$ using Lemma~\ref{lem:drop-connect}, and finally connect $\theta^B_1$ to $\theta^B$ using Lemma~\ref{lem:original_sparse} again.

Finally, our results can be generalized to convolutional networks if we do {\em channel-wise} dropout~\citep{tompson2015efficient, keshari2018guided}.

\begin{remark} \label{rmk:channelwise} For convolutional networks, a {\em channel-wise} dropout will randomly set entire channels to 0 and rescale the remaining channels using an appropriate factor. Theorem~\ref{thm:connect_via_dropout} can be extended to work with channel-wise dropout on convolutional networks.
\end{remark}

%% file: noisestable.tex
\section{Connectivity via noise stability}\label{sec:noise_stable}

In this section, we relate mode connectivity to another notion of robustness for neural networks---noise stability. It has been observed \citep{morcos2018importance} that neural networks often perform as well even if a small amount of noise is injected into the hidden layers. This was formalized in \citep{arora2018stronger}, where the authors showed that noise stable networks tend to generalize well. In this section we use a very similar notion of noise stability, and show that all noise stable solutions can be connected as long as the network is sufficiently overparametrized.

We begin in Section~\ref{sec:noise_stable_model} by restating the definitions of noise stability in \citep{arora2018stronger} and also highlighting the key differences in our definitions. In Section~\ref{sec:experiments} we verify these assumptions in practice. In Section~\ref{sec:noise_stability_drop_out}, 
we first prove that noise stability implies dropout stability (meaning Theorem~\ref{thm:connect_via_dropout} applies) and then show that it is in fact possible to connect noise stable neural networks via even simpler paths than mere dropout stable networks. 


\subsection{Noise stability}\label{sec:noise_stable_model}




First we introduce some additional notations and assumptions. In this section, we consider a finite and fixed training set $S$. For a network parameter $\theta$, the empirical loss function is $L(\theta) = \frac{1}{|S|}\sum_{(x,y)\in S} l(y,f(x))$. Here the loss function $l(y,\hat{y})$ is assumed to be $\beta$-Lipschitz in $\hat{y}$: for any $\hat{y},\hat{y}'\in \R^{h_d}$ and any $y\in\R^{h_d},$ we have $|l(y,\hat{y})-l(y,\hat{y}')|\leq \beta\n{\hat{y}-\hat{y}'}.$
Note that the standard cross entropy loss over the softmax function is $\sqrt{2}$-Lipschitz. 

For any two layers $i\leq j$, let $M^{i,j}$ be the operator for the composition of these layers, such that $x^j = M^{i,j}(x^i)$. Let $J^{i,j}_{x^i}$ be the Jacobian of $M^{i,j}$ at input $x^i$. Since the activation functions are ReLU's, we know $M^{i,j}(x^i) =J^{i,j}_{x^i} x^i$.


\citet{arora2018stronger} used several quantities to define noise stability. We state the definitions of these quantities below. 
\begin{definition}[Noise Stability Quantities] \label{def:allnoise}\label{def:layercushion}
 Given a sample set $S$, the \vocab{layer cushion} of layer $i$ is defined as $\mu_i:=\min_{x\in S}\frac{\|A_i \phi(x^{i-1})\|}{\|A_i\|_F\|\phi(x^{i-1})\|}.$
\label{def:interlayercushion}
	
	For any two layers $i\leq j$, the \vocab{interlayer cushion} $\mu_{i,j}$ is defined as $\mu_{i,j}=\min_{x\in S}        \frac{\|J^{i,j}_{x^i}x^i\|}{\|J^{i,j}_{x^i}\|\|x^i\| }. 
	$
	
	Furthermore, for any layer $i$ the \vocab{minimal interlayer cushion} is defined as\footnote{Note that $J_{x^i}^{i,i} = I_{h_i}$ and $\mu_{i,i} = 1$.} $\icu = \min_{i\leq j\leq d} \mu_{i,j}.$
\label{def:activationcontraction}
	
	The \vocab{activation contraction} $c$ is defined as $c= \max_{x\in S,\ 1\leq i\leq d-1}  \frac{\|x^i\|}{\|\phi(x^i)\|}.$
	\end{definition}

Intuitively, these quantities measures the stability of the network's output to noise for both a single layer and across multiple layers. 
Note that the definition of the {\em interlayer cushion} is slightly different from the original definition in~\citep{arora2018stronger}. Specifically, in the denominator of our definition of interlayer cushion, we replace the Frobenius norm of $J^{i,j}_{x^i}$ by its spectral norm. In the original definition, the interlayer cushion is at most $1/\sqrt{h_i},$ simply because $J_{x^i}^{i,i} = I_{h_i}$ and $\mu_{i,i} = 1/\sqrt{h_i}.$ With this new definition,  the interlayer cushion need not depend on the layer width $h_i$. 




The final quantity of interest is interlayer smoothness, which measures how close the network's behavior is to its linear approximation under noise.
Our focus here is on the noise generated by the dropout procedure (Algorithm~\ref{alg:dropout}). Let $\theta = \{A_1, A_2, ...,A_d\}$ be weights of the original network, and let $\theta^i = \{A_1, \hat{A}_2,\ldots,\hat{A}_i,A_{i+1},\ldots,A_d\}$ be the result of applying Algorithm~\ref{alg:dropout} to weight matrices from layer $2$ to layer $i$.\footnote{Note that $A_1$ is excluded because dropping out columns in $\hat A_2$ already drops out the neurons in layer 1; dropping out columns in $A_1$ would drop out input coordinates, which is not necessary.} For any input $x$, let $\hat{x}^i_i(t)$ and $\hat{x}^i_{i-1}(t)$ be the vector before activation at layer $i$ using parameters $\theta t + \theta^i (1-t)$ and $\theta t + \theta^{i-1} (1-t)$ respectively.

\begin{definition}[Interlayer Smoothness]\label{def:interlayersmoothness}
 Given the scenario above, define \vocab{interlayer smoothness} $\rho$ to be the largest number such that with probability at least $1/2$ over the randomness in Algorithm~\ref{alg:dropout} for any two layers $i,j$ satisfying  for every $2\leq i\leq j\leq d$, $x\in S$, and $0\le t\le 1$ 
	\begin{align*}
	    &\|M^{i,j}(\hat{x}^i_i(t))-J_{x^i}^{i,j}(\hat{x}^i_i(t))\| \leq \frac{\|\hat{x}^i_i(t)-x^i\|\|x^{j}\|}{\rho \|x^{i}\|},\\ &\|M^{i,j}(\hat{x}^i_{i-1}(t))-J_{x^i}^{i,j}(\hat{x}^i_{i-1}(t))\| \leq \frac{\|\hat{x}^i_{i-1}(t)-x^i\|\|x^{j}\|}{\rho \|x^{i}\|}.
	\end{align*}
\end{definition}

If the network is smooth (has Lipschitz gradient), then interlayer smoothness holds as long as $\|\hat{x}^i_i(t)-x^i\|,\|\hat{x}^i_{i-1}(t)-x^i\|$ is small. Essentially the assumption here is that the network behaves smoothly in the random directions generated by randomly dropping out columns of the matrices.

Similar to \citep{arora2018stronger}, we have defined multiple quantities measuring the noise stability of a network. These quantities are in practice small constants as we verify experimentally in Section~\ref{sec:experiments}. Finally, we combine all these quantities to define a single overall measure of the noise stability of a network.

\begin{definition}[Noise Stability] For a network $\theta$ with layer cushion $\mu_i$, minimal interlayer cushion $\icu$, activation contraction $c$ and interlayer smoothness $\rho$, if the minimum width layer $h_{min}$ is at least $\tdomega(1)$ wide, $\rho \ge 3d$ and  $\n{\phi(\hat{x}^i_i(t))}_\infty=O(1/\sqrt{h_i})\n{\phi(\hat{x}^i_i(t))}$ for $1\leq i\leq d-1, 0\leq t\leq 1$, we say the network $\theta$ is $\epsilon$-noise stable for 
$$\epsilon=
\frac{\beta cd^{3/2}\max_{x\in S}(\n{f_{\theta}(x)})}{h_{\min}^{1/2}\min_{2\leq i\leq d}(\mu_i\icu)}.$$ 
\end{definition}

The smaller $\epsilon$, the more robust the network. 
Note that the quantity $\epsilon$ is small as long as the hidden layer width $h_{\min}$ is large compared to the noise stable parameters. Intuitively, we can think of $\epsilon$ as a single parameter that captures the noise stability of the network.

\subsection{Noise stability implies dropout stability}\label{sec:noise_stability_drop_out}
We now show that noise stable local minimizers must also be dropout stable, from which it follows that noise stable local minimizers are connected. We first define the dropout procedure we will be using in Algorithm~\ref{alg:dropout}.


\begin{algorithm}
	\caption{Dropout ($A_i, p$)}
	\begin{algorithmic}[1]
		\REQUIRE Layer matrix $A_i\in \R^{h_i\times h_{i-1}}$, dropout probability $0<p<1$.
		\ENSURE Returns $\hat{A}_i\in \R^{h_i\times h_{i-1}}$.
		\STATE For each $j\in[h_{i-1}],$ let $\delta_j$ be an i.i.d. Bernoulli random variable which takes the value $0$ with probability $p$ and takes the value $\frac{1}{1-p}$ with probability $(1-p)$. 
		\STATE For each $j\in[h_{i-1}],$ let $[\hat{A}_i]_j$ be $\delta_j [A_i]_j$, where $[\hat{A}_i]_j$ and $[A_i]_j$ are the $j$-th column of $\hat{A}_i$ and $A_i$ respectively.   
	\end{algorithmic}
	\label{alg:dropout}
\end{algorithm}


The main theorem that we prove in this section is:
\begin{restatable}{theorem}{noisestableconnect}
\label{thm:noise_stable_connect}
Let $\theta^A$ and $\theta^B$ be two fully connected networks that are both $\epsilon$-noise stable, 
there exists a path with $10$ line segments in parameter space $\pi : [0,1] \to \Theta$ between $\theta^A$ and $\theta^B$ such that\footnote{Here $\tdo(\cdot )$ hides log factors on relevant factors including $|S|, d, \n{x}, 1/\epsilon$ and $h_i \n{A_i}$ for layers $i\in[d]$.} $L(f_{\pi(t)}) \leq \max \{L(f_{\theta^A}),L(f_{\theta^B})\} + \tdo(\epsilon)$ for $0 \leq t \leq 1$.
\end{restatable}

To prove the theorem, we will first show that the networks $\theta^A$ and $\theta^B$ are $\tdo(\epsilon)$-dropout stable. This is captured in the following main lemma:

\begin{restatable}{lemma}{dropoutcrosslayer}
\label{lem:dropout_crosslayer}
Let $\theta$ be an $\epsilon$-noise stable network, 
and let $\theta_1$ be the network with weight matrices from layer $2$ to layer $d$ dropped out by Algorithm~\ref{alg:dropout} with dropout probability $\tdomega(1/h_{min})<p\le \fc 34$.
For any $2\leq i\leq d$, assume $\n{[A_i]_j}=O(\sqrt{p})\n{A_i}_F$ for $1\leq j\leq h_{i-1}.$ For any $0\leq t\leq 1$, define the network on the segment from $\theta$ to $\theta_1$ as $\theta_t:=\theta+t(\theta_1-\theta)$. 
Then, with probability at least $1/4$ over the weights generated by Algorithm~\ref{alg:dropout}, $L(f_{\theta_t}) \le L(f_{\theta})+\tdo(\sqrt{p}\epsilon)$, for any $0\leq t\leq 1$.
\end{restatable}


The main difference between Lemma~\ref{lem:dropout_crosslayer} and Lemma~\ref{lem:original_sparse} is that we can now directly interpolate between the original network and its dropout version, which reduces the number of segments required. This is mainly because in the noise stable setting, we can prove that after dropping out the neurons, not only does the output remains stable but moreover {\em every} intermediate layer also remains stable.

From Lemma~\ref{lem:dropout_crosslayer}, the proof of Theorem~\ref{thm:noise_stable_connect} is very similar to the proof of Theorem~\ref{thm:connect_via_dropout}. The detailed proof is given in Section~\ref{sec:noise_stability_proof}.

The additional power of Lemma~\ref{lem:dropout_crosslayer} also allows us to consider a smaller dropout probability. The theorem below allows us to trade the dropout fraction with the energy barrier $\epsilon$ that we can prove---if the network is highly overparametrized, one can choose a small dropout probability $p$ which allow the energy barrier $\epsilon$ to be smaller. 

\begin{restatable}{theorem}{teacherstudent}\label{cor:teacher_student}
Suppose there exists a network $\theta^*$ with layer width $h_{i}^*$ for each layer $i$ that achieves loss $L(f_{\theta^*})$, and minimum hidden layer width $h^*_{min} = \tdomega(1)$.
Let $\theta^A$ and $\theta^B$ be two $\epsilon$-noise stable networks. 
For any dropout probability $1.5\max_{1\leq i\leq d-1}(h_i^*/h_i)\leq p\leq 3/4$, if for any $2\leq i\leq d$, $1\leq j\leq h_{i-1},$  $\n{[A_i]_j}=O(\sqrt{p})\n{A_i}_F$ then 
there exists a path with $13$ line segments in parameter space $\pi : [0,1] \to \Theta$ between $\theta^A$ and $\theta^B$ such that $L(f_{\pi(t)}) \leq \max \{L(f_{\theta^A})+\tdo(\sqrt{p}\epsilon),L(f_{\theta^B})+\tdo(\sqrt{p}\epsilon), L(f_{\theta^*})\}$ for $0 \leq t \leq 1$.
\end{restatable}

Intuitively, we prove this theorem by connecting $\theta^A$ and $\theta^B$ via the neural network $\theta^*$ with narrow hidden layers. The detailed proof is given in Section~\ref{sec:noise_stability_proof}.

%% file: examples_main.tex
\section{Disconnected modes in two-layer nets}
\label{sec:example:main}\label{sec:examples:main}
The mode connectivity property is not true for every neural network. \citet{freeman2016topology} gave a counter-example showing that if the network is not overparametrized, then there can be different global minima of the neural network that are not connected. \citet{venturi2018spurious} showed that spurious valleys can exist for 2-layer ReLU nets with an arbitrary number of hidden units, but again they do not extend their result to the overparametrized setting. In this section, we show that even if a neural network is overparametrized---in the sense that there exists a network of smaller width that can achieve optimal loss---there can still be two global minimizers that are not connected.

In particular, suppose we are training a two-layer ReLU student network with $h$ hidden units to fit a dataset generated by a ground truth two-layer ReLU teacher network with $h_t$ hidden units such that the samples in the dataset are drawn from some input distribution and the labels computed via forward passes through the teacher network. The following theorem demonstrates that regardless of the degree to which the student network is overparametrized, we can always construct such a dataset for which global minima are not connected.


\begin{theorem}\label{thm:counterexample}
For any width $h$ and and convex loss function $l: \R \times \R \mapsto \R$ such that $l(y,\hat y)$ is minimized when $y = \hat y$, there exists a dataset generated by ground-truth teacher network with two hidden units (i.e. $h_t = 2$) and one output unit such that global minimizers are not connected for a student network with $h$ hidden units.
\end{theorem}

Our proof is based on an explicit construction. The detailed construction is given in Section~\ref{sec:example}.

%% file: experiments.tex
\section{Experiments}
\label{sec:experiments}
We now demonstrate that our assumptions and theoretical findings accurately characterize mode connectivity in practical settings. In particular, we empirically validate our claims using standard convolutional architectures---for which we treat individual filters as the hidden units and apply channel-wise dropout (see Remark~\ref{rmk:channelwise})---trained on datasets such as CIFAR-10 and MNIST.

Training with dropout is not necessary for a network to be either dropout-stable or noise-stable. Recall that our definition of dropout-stability merely requires the existence of a particular sub-network with half the width of the original that achieves low loss. Moreover, as Theorem~\ref{cor:teacher_student} suggests, if there exists a narrow network that achieves low loss
, then we need only be able to drop out a number of filters equal to the width of the narrow network to connect local minima.

\begin{figure}[!htb]
\minipage{0.32\textwidth}
  \includegraphics[width=1\linewidth]{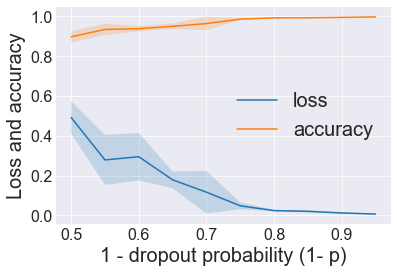}
\endminipage\hfill
\minipage{0.32\textwidth}
  \includegraphics[width=1\linewidth]{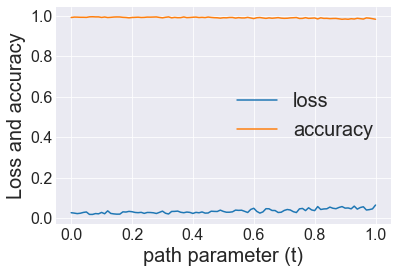}
\endminipage\hfill
\minipage{0.32\textwidth}%
  \includegraphics[width=1\linewidth]{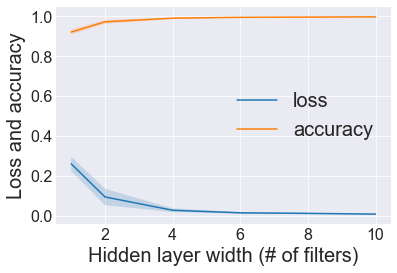}
\endminipage
\caption{Results for convolutional networks trained on MNIST.}
\label{fig:mnist}
\end{figure}

First, we demonstrate in the left plot in Figure~\ref{fig:mnist} on MNIST that 3-layer convolutional nets (not counting the output layer) with 32 $3 \times 3$ filters in each layer tend to be fairly dropout stable---both in the original sense of Definition~\ref{def:drop-stability} and especially if we relax the definition to allow for wider subnetworks---despite the fact that no dropout was applied in  training. For each trial, we randomly sampled $20$ dropout networks with $\emph{exactly}$ $\lfloor 32(1-p) \rfloor$ non-zero filters in each layer and report the performance of the best one. In the center plot, we verify for $p = 0.2$ we can construct a linear path $\pi(t):\R \to \Theta$ from our convolutional net to a dropout version of itself. 
Similar results were observed when varying $p$. Finally, in the right plot we demonstrate the existence of 3-layer convolutional nets just a few filters wide that are able to achieve low loss on MNIST. Taken together, these results indicate that our path construction in Theorem~\ref{cor:teacher_student} performs well in practical settings. In particular, we can connect two convolutional nets trained on MNIST by way of first interpolating between the original nets and their dropped out versions with $p = 0.2$, and then connecting the dropped out versions by way of a narrow subnetwork with at most $\lfloor 32p \rfloor$ non-zero filters.

\begin{figure}[!th]
\begin{minipage}[b]{0.8\textwidth}
	\begin{subfigure}[t]{0.27\textwidth}
	\centering
	\includegraphics[height=2.9cm, width=3.5cm]{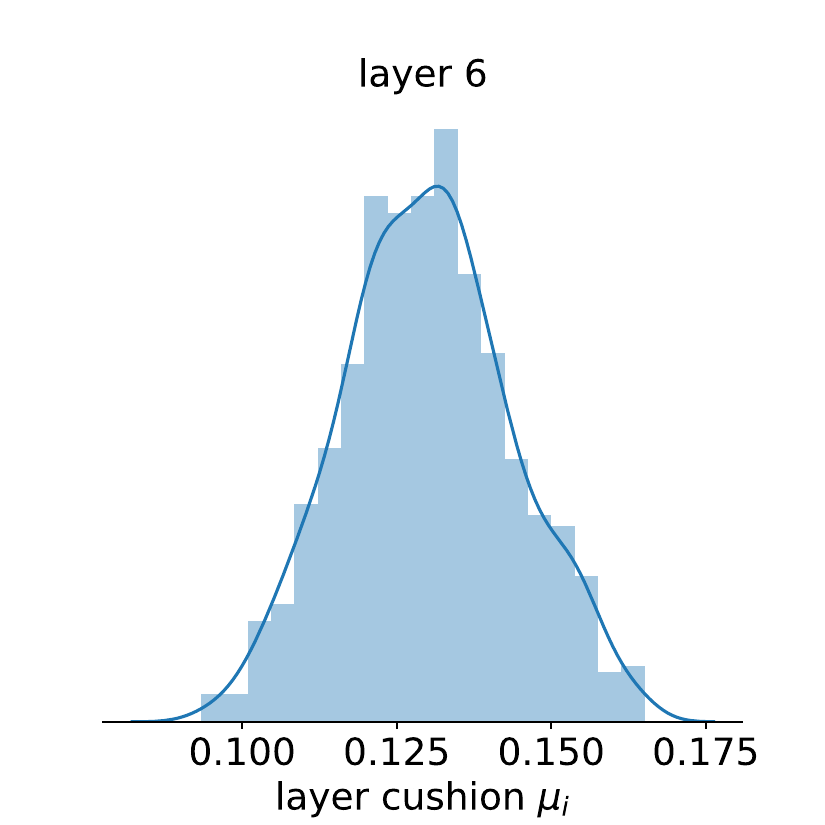}
	\end{subfigure}
	\begin{subfigure}[t]{0.27\textwidth}
	\centering
	\includegraphics[height=2.9cm, width=3.5cm]{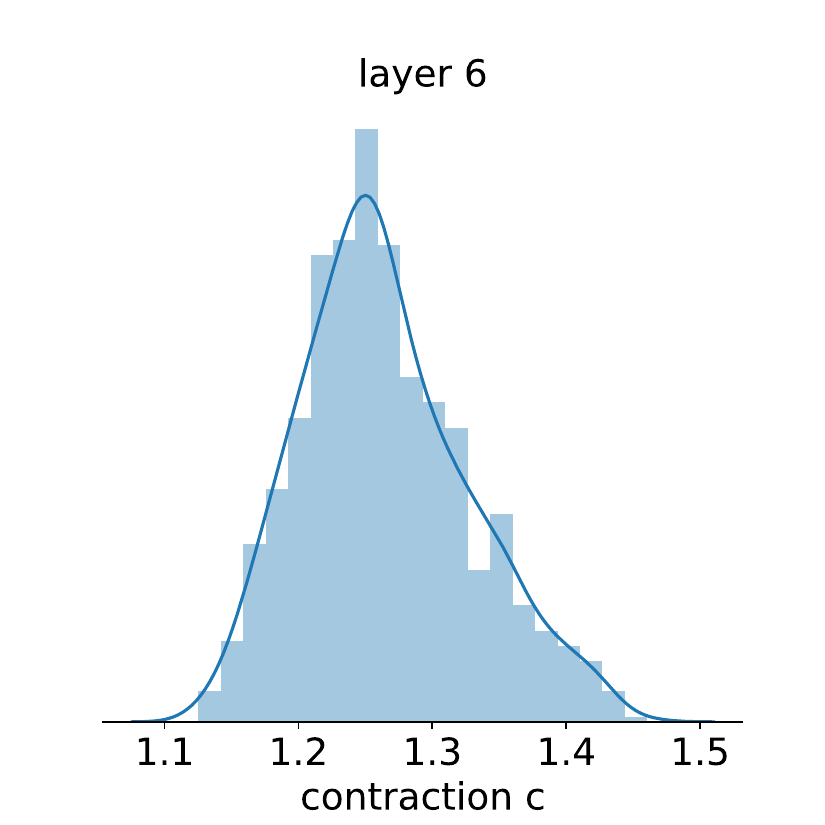}
	\end{subfigure}
    
	\begin{subfigure}[t]{0.27\textwidth}
	\centering
	\includegraphics[height=2.9cm, width=3.5cm]{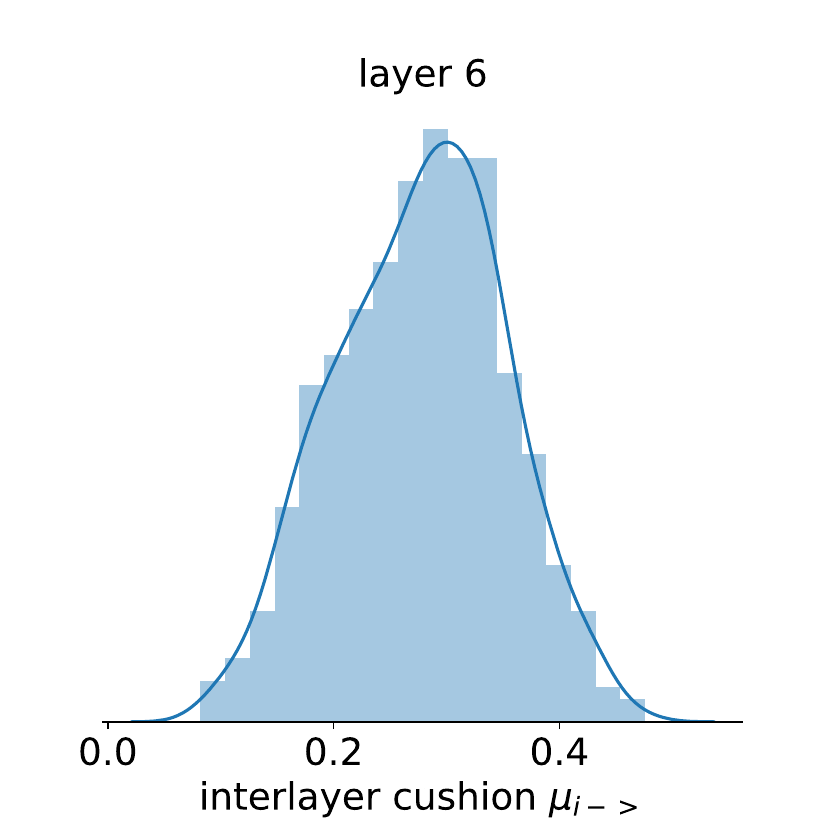}
	\end{subfigure}
	\begin{subfigure}[t]{0.27\textwidth}
	\centering
	\includegraphics[height=2.9cm, width=3.5cm]{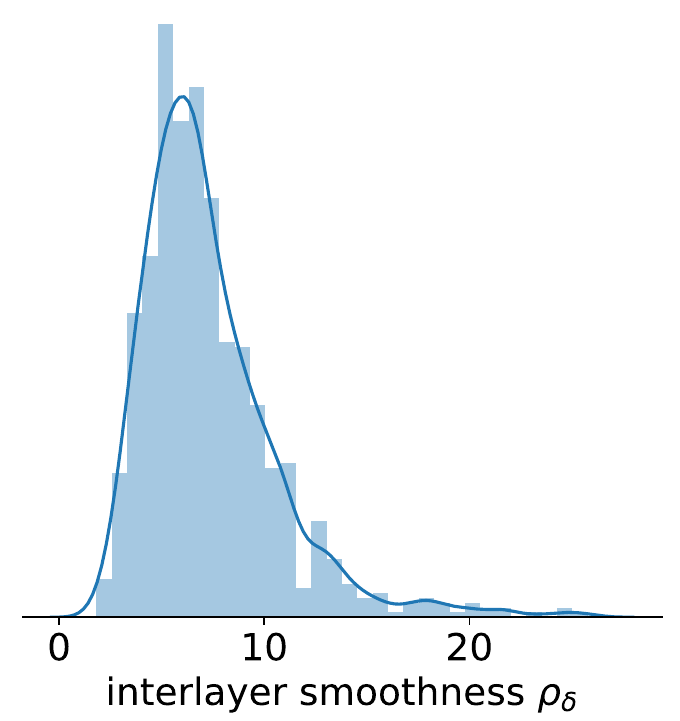}
	\end{subfigure}
\end{minipage}
\hspace{-4.5cm}
\begin{minipage}[b]{0.1\textwidth}
    \centering
    \begin{subfigure}[t]{0.8\textwidth}
	\centering
	\includegraphics[height=5cm, width=7cm]{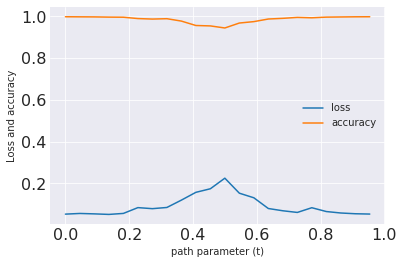}
	\end{subfigure}
\end{minipage}
\caption{Left) Distribution of layer cushion, activation contraction, interlayer cushion and interlayer smoothness of the 6-th layer of a VGG-11 network on the training set. The other layers' parameters are exhibited in Section~\ref{appdix:verification}. Right) The loss and training accuracy along the path between two noise stable VGG-11 networks described in Theorem~\ref{lem:dropout_crosslayer}.}
\label{fig:properties}
\end{figure}


We also demonstrate that the VGG-11~\citep{simonyan2014very} architecture trained with channel-wise dropout~\citep{tompson2015efficient, keshari2018guided} with $p=0.25$ at the first three layers\footnote{we find the first three layers are less resistant to channel-wise dropout.}
and $p=0.5$ at the others on CIFAR-10 converges to a noise stable minima---as measured by layer cushion, interlayer cushion, activation contraction and interlayer smoothness. The network under investigation achieves 95\% training and 91\% test accuracy with channel-wise dropout \emph{activated}, in comparison to 99\% training and 92\% test accuracy with dropout turned off. Figure~\ref{fig:properties} plots the distribution of the noise stability parameters over different data points in the training set, from which we can see they behave nicely. Interestingly, we also discovered that networks trained without channel-wise dropout exhibit similarly nice behavior on all but the first few layers. Finally, in Figure~\ref{fig:properties}, we demonstrate that the training loss and accuracy obtained via the path construction in Theorem~\ref{lem:dropout_crosslayer} between two noise stable VGG-11 networks $\theta_A$ and $\theta_B$ remain fairly low and high respectively---particularly in comparison to directly interpolating between the two networks, which incurs loss as high as 2.34 and accuracy as low as 10\%, as shown in Section~\ref{sec:direct_interpolation}.

Further details on all experiments are provided in Section~\ref{appdix:hyperparameters}.

%% file: dropout_proofs.tex
\section{Proofs for connectivity of dropout-stable optima}\label{sec:dropout_proof}
\begin{proofof}{Lemma~\ref{lem:original_sparse}}
Without loss of generality, suppose for each $\theta_i$ that the subset of $\lfloor h_i/2 \rfloor$ non-zero hidden units in each layer are all indexed between $1$ and $\lfloor h_i/2 \rfloor$. For $1 < i < d$, we can partition $A_i$ into quadrants such that $A_i=\left[\begin{array}{c|c}
L_i & C_i \\
\hline
D_i & R_i
\end{array}\right]$. 
(Here, $L_i\in \R^{\fl{h_i/2}\times \fl{h_i/2}}$. If $h_i$ is odd, when we write $L_i$ in the other quadrants we implicitly pad it with zeros in a consistent manner.)
Similarly, we can partition $A_1$ such that $A_1 = 
\left[\begin{array}{c}
L_1 \\
\hline
B_1
\end{array}\right]$ and $A_d$ such that $A_d = \left[\begin{array}{c|c}
L_d & R_d \\
\end{array}\right]$.  We will sometimes use the notation $A_i$ to refer to the value of $A_i$ at a given point on our path, while $A_i^\theta$ will always refer to the value of $A_i$ at $\theta$. We now proceed to prove via induction the existence of a path from $\theta$ to $\theta_i$ for all $i$ whose loss is bounded by $L(f_\theta)+\epsilon$, from which the main result immediately follows.\\

\noindent\textbf{Base case: from $\theta$ to $\theta_{d-1}$}
As a base case of the induction, we need to construct a path from $\theta$ to $\theta_{d-1}$, such that the loss is bounded by $L(f_\theta)+\epsilon$. First, note that setting a particular subset of columns (e.g. the right half of columns) in $A_i$ to zero is equivalent to setting the corresponding rows (e.g. the bottom half of rows) of $A_{i-1}$ to zero. So from the fact that $L(f_{\theta_{d-1}}) \leq L(f_\theta)+\epsilon$ it follows that we can equivalently replace $A_d^{\theta}$ with $\left[\begin{array}{c|c}
r L_d^{\theta} & 0 \\
\end{array}\right]$ without increasing our loss by more than $\epsilon$. 

In fact, because our loss function is convex over $A_d$ we can actually interpolate $A_d$ between $A_d^{\theta}$ and $\left[\begin{array}{c|c}
r L_d^{\theta} & 0 \\
\end{array}\right]$ while keeping our loss below $L(f_\theta)+\epsilon$ at every point along this subpath.

Then, because $R_d = 0$ we can modify both $D_{d-1}$ and $R_{d-1}$ any way we'd like without affecting the output of our network. In particular, we can interpolate $A_{d-1}$ between $A_{d-1}^{\theta}$ and $\left[\begin{array}{c|c}
L_{d-1}^{\theta} & C_{d-1}^{\theta} \\
\hline
0 & 0
\end{array}\right]$ while keeping our loss constant long this subpath, thus arriving at $\theta_{d-1}$.

\noindent \textbf{From $\theta_k$ to $\theta_{k-1}$}

Suppose we have found a path from $\theta$ to $\theta_k$ such that (1) $A_d^{\theta_k} = \left[\begin{array}{c|c}
r L_d^{\theta} & 0 \\
\end{array}\right]$, (2) $A_i^{\theta_k} = \left[\begin{array}{c|c}
r L_{i}^{\theta} & 0 \\
\hline
0 & 0
\end{array}\right]$ for $k < i < d$, (3) $A_k^{\theta_k} = \left[\begin{array}{c|c}
L_{k}^{\theta} & C_{k}^{\theta} \\
\hline
0 & 0
\end{array}\right]$, and (4) $A_i^{\theta_k} = A_i^{\theta}$ for $i < k$, such that the loss along the path is at most $L(f_{\theta}) + \epsilon$. Note that $\theta_{d-1}$ satisfies all these assumptions, including in particular (2) as there are of course no $A_i$ between $A_{d-1}$ and $A_{d}$. Now let us extend this path to $\theta_{k-1}$. 

First, because the rightmost columns of $A_i$ are zero for $k < i \leq d$, we can modify the bottom rows of $A_i$ for $k \leq i < d$ without affecting the output of our network. In particular, we can set $A_k$ to $\left[\begin{array}{c|c}
L_k^{\theta} & C_k^{\theta} \\
\hline
r L_k^{\theta} & 0
\end{array}\right]$, as well as $A_i$ to $\left[\begin{array}{c|c}
r L_i^{\theta} & 0 \\
\hline
0 & r L_i^{\theta}
\end{array}\right]$ for $k < i < d$. From the fact that the loss is convex over $A_d$ and that $L(f_{\theta_{k-1}}) < L(f_{\theta}) + \epsilon$, it then follows that we can set $A_d$ to $\left[\begin{array}{c|c}
0 & r L_d^{\theta} \\
\end{array}\right]$ via interpolation while keeping our loss below $L(f_{\theta}) + \epsilon$. In particular, note that because the off-diagonal blocks of $A_i$ are zero for $k < i < d$, interpolating between the leftmost columns of $A_d$ being non-zero and the rightmost columns of $A_d$ being non-zero simply amounts to interpolating between the outputs of the two subnetworks comprised respectively of the first $\lfloor h_i/2 \rfloor$ and last $\lfloor h_i/2 \rfloor$ rows of $A_i$ for $k \leq i < d$.

Once we have the leftmost columns of $A_d$ set to zero and $A_i$ in block-diagonal form for $k < i < d$, we can proceed to modify the top rows of $A_k$ however we'd like without affecting the output of our network. Specifically, let us set $A_k$ to $\left[\begin{array}{c|c}
r L_k^{\theta} & 0 \\
\hline
r L_k^{\theta} & 0
\end{array}\right]$. We can then reset $A_d$ to $\left[\begin{array}{c|c}
r L_d^{\theta} & 0 \\
\end{array}\right]$ via interpolation---this time without affecting our loss since the weights of our two subnetworks are equivalent---and afterwards set $D_k$ to zero and $R_i$ to zero for $k \leq i < d$---again without affecting our loss since the rightmost columns of $A_d$ are now zero, meaning that the bottom rows of $A_i$ have no affect on our network's output.

Following these steps, we will have $A_i = \left[\begin{array}{c|c}
r L_i^{\theta} & 0 \\
\hline
0 & 0
\end{array}\right]$ for $k \leq i < d$ and $A_d = \left[\begin{array}{c|c}
r L_d^{\theta} & 0 \\
\end{array}\right]$. And so we are now free to set the bottom rows of $A_{k-1}$ to zero without affecting our loss, thus arriving at $\theta_{k-1}$.
\end{proofof}

\begin{lemma}\label{lem:permutation}
Let $\theta$ be a parameter such that at least $\lceil h_i/2 \rceil$ of the units in each hidden layer have been set to zero. Then we can achieve an arbitrary permutation of the non-zero hidden units of $\theta$ via a path consisting of just 5 line segments such that our loss is constant along this path.
\end{lemma}
\begin{proof}
Let $\pi: [h_i] \mapsto [h_i]$ be some permutation over the units in layer $i$. Without loss of generality, suppose all  non-zero units in layer $i$ are indexed between $0$ and $\lfloor h_i/2 \rfloor$, and define $\pi':\left[\lfloor h_i/2 \rfloor\right] \mapsto [h_i] \setminus \left[\lfloor h_i/2 \rfloor\right]$ as any one-to-one mapping such that $\pi'(i) = \pi(i)$ if $\pi(i) \in [h_i] \setminus \left[\lfloor h_i/2 \rfloor\right]$. Note that when we refer to a unit $j$ as ``set to zero'', we mean that both row $j$ of $A_i$ and column $j$ of $A_{i+1}$ have been set to zero.

To permute the units of layer $i$, we can first simultaneously copy the non-zero rows of $A_i$ into a subset of the rows that have been set to zero. Specifically, for $j \in \left[\lfloor h_i/2 \rfloor\right]$ we can copy row $j$ of $A_i$ into row $\pi'(j)$ via interpolation and without affecting our loss, due to the fact that column $\pi'(j)$ in $A_{i+1}$ is set to zero. We can then set column $j$ of $A_{i+1}$ to zero while copying its value to column $\pi'(j)$, again via interpolation and without affecting our loss since rows $j$ and $\pi'(j)$ of $A_i$ are now equivalent.

Following these first two steps, the first $\lfloor h_i/2 \rfloor$ columns of $A_{i+1}$ will have been set to zero. Thus, for all $j \in \left[\lfloor h_i/2 \rfloor\right]$ such that $\pi(j) \in [h_i/2]$ we can copy row $\pi'(j)$ of $A_i$ into row $\pi(j)$ without affecting our loss. We can then set column $\pi'(j)$ of $A_{i+1}$ to zero while copying its value into column $\pi(j)$ via interpolation and without affecting our loss since rows $\pi'(j)$ and $\pi(j)$ of $A_i$ are now equivalent. Setting row $\pi'(j)$ to zero---again for all $j \in \left[\lfloor h_i/2 \rfloor\right]$ such that $\pi(j) \in [h_i/2]$---completes the permutation for layer $i$.

Note that because we leave the output of layer $i$ unchanged throughout the course of permuting the units of layer $i$, it follows that we can perform all swaps across all layers simultaneously. And so from the fact that permuting each layer can be done in 5 steps---each of which consists of a single line segment in parameter space---the main result immediately follows.
\end{proof}

\begin{proofof}{Lemma~\ref{lem:drop-connect}}
Without loss of generality, suppose for $\theta$ that the subset of $\lfloor h_i/2 \rfloor$ non-zero hidden units in each layer $i$ are all indexed between $0$ and $\lfloor h_i/2 \rfloor$. Note that when we refer to a unit as ``set to zero", we mean that both the corresponding row of $A_i$ and column of $A_{i+1}$ have been set to zero. Adopting our notation in Lemma \ref{lem:original_sparse}, we can construct a path from $\theta$ to $\theta'$ as follows.

First, from the fact that the second half of units in each hidden layer $i$ have been set to zero in $\theta$ we have that $A_1^{\theta} = \left[\begin{array}{c}
L_1^{\theta} \\
\hline
0
\end{array}\right]$,  $A_i^{\theta} = \left[\begin{array}{c|c}
L_i^{\theta} & 0 \\
\hline
0 & 0
\end{array}\right]$ for $1 < i < d$, and $A_d^{\theta} = \left[\begin{array}{c|c}
L_d^{\theta} & 0 \\
\end{array}\right]$. Similarly, half the rows of $A_1^{\theta'}$ are zero, half the rows and columns of $A_i^{\theta'}$ are zero for $1 < i < d$, and half the columns of $A_d^{\theta'}$ are zero. Note that the indices of the non-zero units in $\theta'$ may intersect with those of the non-zero units in $\theta$. For $1 \leq i \leq d$, let $B_i$ denote the submatrix of $A_i$ corresponding to the non-zero rows and columns of $A_i^{\theta'}$.

Because $A_i^{\theta}$ are block-diagonal for $1 < i < d$ and the rightmost columns of $A_d^{\theta}$ are zero, starting from $\theta$ we can modify the bottom rows of $A_i$ for $1 \leq i < d$ any way we'd like without affecting our loss---as done in our path construction for Lemma \ref{lem:original_sparse}. In particular, let us set $A_i$ to $\left[\begin{array}{c|c}
L_i^{\theta} & 0 \\
\hline
0 & B_i^{\theta'}
\end{array}\right]$ for $1 < i < d$ and $A_1$ to $\left[\begin{array}{c}
L_1^{\theta} \\
\hline
B_1^{\theta'}
\end{array}\right]$. Then, from the fact that our loss function is convex over $A_d$ it follows that we can set $A_d$ to $\left[\begin{array}{c|c}
0 & B_d^{\theta'} \\
\end{array}\right]$ via interpolation while keeping our loss below $\max\{L(f_{\theta}),L(f_{\theta'})\}$. Finally, from the fact that the leftmost columns of $A_d$ are now zero and $A_i$ are still block-diagonal for $1 < i < d$, it follows that we can set $L_i$ to zero for $1 \leq i < d$ without affecting our loss---thus making $A_i$ equal to $\left[\begin{array}{c|c}
0 & 0 \\
\hline
0 & B_i^{\theta'}
\end{array}\right]$ for $1 < i < d$ and $A_1$ equal to $\left[\begin{array}{c}
0 \\
\hline
B_1^{\theta'}
\end{array}\right]$.

To complete our path from $\theta$ to $\theta'$ we now simply need to permute the units of each hidden layer so as to return the elements of $B_i^{\theta'}$ to their original positions in $A_i$ for each $i$. From Lemma~\ref{lem:permutation} it follows that we can accomplish this permutation via 5 line segments in parameter space without affecting our loss. Combined with the previous steps above, we have constructed path from $\theta$ to $\theta'$ consisting of a total of 8 line segments whose loss is bounded by $\max\{L(f_{\theta}),L(f_{\theta'})\}$.
\end{proofof}

\begin{proofof}{Theorem~\ref{thm:connect_via_dropout}}
First, from Lemma \ref{lem:original_sparse} we know we can construct paths from both $\theta^A$ to $\theta_1^A$ and $\theta^B$ to $\theta_1^B$ while keeping our loss below $L(f_{\theta^A})+\epsilon$ and $L(f_{\theta^B})+\epsilon$ respectively. From Lemma \ref{lem:drop-connect} we know that we can construct a path from $\theta_1^A$ to $\theta_1^B$ such that the loss along the path is bounded by $\max\{L(f_{\theta_1^A}),L(f_{\theta_1^B})\}$. The main result then follows from the fact that $L(f_{\theta_1^A}) \leq L(f_{\theta^A})+\epsilon$ and $L(f_{\theta_1^B}) \leq L(f_{\theta^B})+\epsilon$ due to $\theta^A$ and $\theta^B$ both being $\epsilon$-dropout stable.
\end{proofof}

%% file: noisestability_proof.tex
\section{Proofs for connectivity via noise stability}\label{sec:noise_stability_proof}
In this section, we give detailed proofs showing that noise stability implies connectivity. In the following lemma, we first show that the network output is stable if we randomly dropout columns in a single layer using Algorithm~\ref{alg:dropout}.  

\begin{lemma}\label{lem:dropout_layer}
For any layer $2\leq i\leq d$, let $G=\{(U^{(l)},x^{(l)})\}_{l=1}^m$ be a set of matrix/vector pairs of size $m$ where $U\in\R^{h_d \times h_i}$ and $x\in\R^{h_{i-1}}$ satisfying $\n{x}_{\infty} = O\pa{\frac{\n{x}}{\sqrt{h_{i-1}}}}$. Given $A_i$, let $\hat{A}_i\in \R^{h_{i}\times h_{i-1}}$ be the output of Algorithm~\ref{alg:dropout} with dropout probability $0<p\le\fc 34$. Assume $\n{[A_i]_j}=O(\sqrt{p})\n{A_i}_F$ for $1\leq j\leq h_{i-1}.$ Given any $0<\delta<1$, let $\epsilon'=O\pa{\sqrt{\frac{\log(mh_d/\delta)p}{h_{\min}}} }$, with probability at least $1-\delta$, we have for any $(U,x)\in G$ that $\|U (\hat{A}_i - A_i) x\| \leq \eps'\|A_i\|_F\|U\|\|x\|$. Further assuming $h_{\min}=\Omega\pa{\frac{\log(1/\delta)}{p}}$, we know with probability at least $1-\delta,$ no less than $\frac{2}{3}p$ fraction of columns in $\hat{A}_i$ are zero vectors.
\end{lemma}

Intuitively, this lemma upper-bounds the change in the network output after dropping out a single layer. In the lemma, we should think of $x$ as the input to the current layer, $A_i$ as the layer matrix and $U$ as the Jacobian of the network output with respect to the layer output. If the activation pattern does not change after the dropping out, $U\hat{A}_i x$ is exactly the output of the dropped out network and $\|U (\hat{A}_i - A_i) x\|$ is the change in the network output.

\begin{proofof}{Lemma~\ref{lem:dropout_layer}}
Fixing $2\leq i\leq d$ and one pair $(U,x)\in G$, we show with probability at least $1-\frac{\delta}{m}$, $\|U (\hat{A}_i - A_i) x\| \leq \eps'\|A_i\|_F\|U\|\|x\|$.
Let $U_k$ be the $k$-th column of $U$. Then by definition of $\hat A_i$ in the algorithm, we know 
\begin{align*}
U(\hat{A}_i-A_i)x & = \sum_{k,j}U_k [A_i]_{kj} x_j (\delta_j-1)\\
&= \sum_{j}\pa{\sum_k U_k [A_i]_{kj}} x_j (\delta_j-1), 
\end{align*}
where $\delta_j$ is an i.i.d. Bernoulli random variable which takes the value $0$ with probability $p$ and takes the value $\frac{1}{1-p}$ with probability $(1-p)$.

Let $[A_i]_j$ be the $j$-th column of $A_i$. Because $p\le \fc 34$, $\rc{1-p}=O(1)$ (any $p$ bounded away from 1 will work). Hence the norm for each individual term can be bounded as follows. 
\begin{align*}
\ve{\pa{\sum_k U_k [A_i]_{kj}} x_j (\delta_j-1)}&\stackrel{(*)}\leq O\pf{\n{x}}{\sqrt{h_{i-1}}}\n{U[A_i]_j}\\
&\leq O\pf{\n{x}}{\sqrt{h_{\min}}} \n{U}\n{[A_i]_j}\\
&\stackrel{(\dagger)}{\leq} O\pf{\sqrt{p}\n{U}\n{A_i}_F\n{x}}{\sqrt{h_{\min}}},
\end{align*}
where (*) uses the assumption that $\n{x}_{\infty} = O\pf{\n{x}}{\sqrt{h_{i-1}}}$ and $(\dagger)$ holds because we assume $\n{[A_i]_j}=O(\sqrt{p})\n{A_i}_F$ for $1\leq j\leq h_{i-1}.$

For the total variance, we have 
\begin{align*}
\sigma^2 :&= \sum_j \E\ba{\ve{ \pa{\sum_k U_k [A_i]_{kj}} x_j (\delta_j-1)}^2}\\
&\leq \sum_j \ns{U[A_i]_j}|x_j|^2\pa{(0-1)^2\times p + \pa{\frac{1}{1-p}-1}^2\times (1-p) }\\
&\stackrel{(*)}= \sum_j \ns{U[A_i]_j}\cdot O\pf{\ns{x}}{h_{i-1}}\cdot p\pa{1+\frac{p}{1-p}}\\
&\leq \ns{UA_i}_F\cdot O\pf{\ns{x}}{h_{\min}}\cdot p\\
&\leq O\pf{p\ns{U}\ns{A_i}_F\ns{x}}{h_{\min}},
\end{align*}
where inequality $(*)$ uses the assumption that $\n{x}_{\infty} = O\pf{\n{x}}{\sqrt{h_{i-1}}}$. Then, by the vector Bernstein inequality (Lemma~\ref{lm:vectorBernstein}), we know given $0<\delta<1$, there exists $\epsilon'=O\pa{\sqrt{\frac{p \log(mh_d/\delta)}{h_{\min}}} }$, with probability at least $1-\frac{\delta}{m}$, we have 
$$\|U (\hat{A}_i - A_i) x\| \leq \eps'\|A_i\|_F\|U\|\|x\|.$$

Taking the union bound over all $(U,x)$ pairs in $G$, we know that with probability at least $1-\delta$, for any $(U,x)\in G$, $\|U (\hat{A}_i- A_i) x\| \leq \eps'\|A_i\|_F\|U\|\|x\|.$ 

Suppose $h_{\min}=\Omega\pa{\frac{\log(1/\delta)}{p}}$; 
then by the Chernoff bound, we know with probability at least $1-\delta,$ the dropped out fraction is at least $\frac{2}{3}p$. Taking another union bound concludes our proof.
\end{proofof}

Now we are ready to prove Lemma~\ref{lem:dropout_crosslayer}. The idea is similar to \citep{arora2018stronger}, but we give the proof here for completeness.
\dropoutcrosslayer*

\begin{proofof}{Lemma~\ref{lem:dropout_crosslayer}}
We first bound the difference between the dropped out network $\theta_1$ and the original network $\theta$.

\paragraph{Bounding $\|f_\theta(x) - f_{\theta_1}(x)\|$:} We first show that with probability at least $1/2-\delta,$ $\|f_\theta(x) - f_{\theta_1}(x)\| =\|x^d - \hat x^d_d\| \le \eps'\|f_{\theta}(x)\|,$ where $\eps'$ will be specified later.
For any layer $i\geq 1$ and letting $\hat{x}^j_i$ be the vector before activation at layer $j$ if the weights $A_2,\dots,A_i$ are replaced by $\hat{A}_2,\dots,\hat{A}_i$.

According to Lemma~\ref{lem:dropout_layer}, for any layer $2\leq i\leq d$, given $0<\delta<1$, let $\epsilon'=O\pa{\sqrt{\frac{p c^2d^2\log(md h_d/\delta)}{h_{\min}\min\limits_{2\leq i\leq d}(\mu_i^2\icu^2) }}}$, with probability at least $1-\delta/d$ over $\hat{A}_i,$ we have 
\begin{equation}
    \|U (\hat{A}_i - A_i) x\| \leq \frac{\eps'\mu_i\icu}{6cd}\|A\|_F\|U\|\|x\|\label{eq:noise_layer}
\end{equation}
for any $(U,x)\in \{(J^{i,j}_{x^i},\phi(\hx^{i-1}_{i-1}))|x\in S,i\leq j\leq d\}$. By taking a union bound over $i$, we know inequality~(\ref{eq:noise_layer}) holds with probability at least $1-\delta$ for every $i.$ Recall that the interlayer smoothness holds with probability at least $1/2.$ Taking another union, we know with probability at least $1/2-\delta,$ interlayer smoothness holds and inequality~(\ref{eq:noise_layer}) holds for every $2\leq i\leq d.$ Next, conditioning on the success of these two events, we will inductively prove for any $1\leq i\leq d$, for any $i\leq j\leq d,$
	$$
	\|\hat{x}^j_i - x^j\|\le (i/d)\eps'\|x^j\|.
	$$
	
	For the base case $i = 1$, since we are not dropping out any weight matrix, the inequality is trivial. For any $1\leq i-1\leq d-1,$ suppose $\|\hat{x}^j_{i-1} - x^j\|\le \frac{i-1}{d}\eps'\|x^j\|$ for any $i-1\leq j\leq d$; we prove the induction hypothesis holds for layer $i.$
	
	For any $i\leq j\leq d$ we have
	\begin{equation*}
	\|\hat{x}^j_i - x^j\| = \|(\hat{x}^j_i - \hat{x}^j_{i-1})+(\hat{x}^j_{i-1} - x^j)\| \leq \|\hat{x}^j_i - \hat{x}^j_{i-1}\| +\|\hat{x}^j_{i-1} - x^j\|.
	\end{equation*}

By the induction hypothesis, we know the second term can be bounded by $(i-1)\eps'\|x^j\|/d$. Therefore, in order to complete the induction step, it suffices to show that the first term is bounded by $\eps'\|x^j\|/d$. For simplicity, we also denote $\hat{x}^{i-1}_{i-1}$ as $\hat{x}^{i-1}$.
Let $\Delta_i = \hat{A}_i - A_i$. 
We can decompose the error into two terms:
	\begin{align}
	\nonumber
	\|\hat{x}^j_i - \hat{x}^j_{i-1}\|
	& = \|M^{i,j}(\hat{A}_i \phi(\hat{x}^{i-1})) - M^{i,j}(A_i \phi(\hat{x}^{i-1}))\|\\
	\nonumber
	& = \|M^{i,j}(\hat{A}_i \phi(\hat{x}^{i-1})) - M^{i,j}(A_i \phi(\hat{x}^{i-1})) + J^{i,j}_{x^i}(\Delta_i \phi(\hat{x}^{i-1})) - J^{i,j}_{x^i}(\Delta_i \phi(\hat{x}^{i-1}))\| \\
	& \leq \|J^{i,j}_{x^i}(\Delta_i \phi(\hat{x}^{i-1}))\| + \|M^{i,j}(\hat{A}_i \phi(\hat{x}^{i-1})) - M^{i,j}(A_i \phi(\hat{x}^{i-1})) - J^{i,j}_{x^i}(\Delta_i \phi(\hat{x}^{i-1}))\| 
	\label{e:induct-error}
	\end{align}
	The first term of~\eqref{e:induct-error} can be bounded as follows:
	\begin{align}
	\nonumber
	& \|J^{i,j}_{x^i}\Delta_i \phi(\hat{x}^{i-1})\|\\
		\nonumber
	& \le (\eps'\mu_i\icu/6cd)\|J^{i,j}_{x^i}\|\|A_i\|_F\|\phi(\hat{x}^{i-1})\|&& \mbox{Lemma~\ref{lem:dropout_layer}}\\
		\nonumber
	& \le (\eps'\mu_i\icu/6cd)\|J^{i,j}_{x^i}\|\|A_i\|_F\|\hat{x}^{i-1}\|&& \mbox{$\phi$ (ReLU) is 1-Lipschitz}\\
		\nonumber
	&\le (\eps'\mu_i\icu/3cd)\|J^{i,j}_{x^i}\|\|A_i\|_F\|x^{i-1}\|&& \mbox{Induction hypothesis,}\\
	\nonumber
	&&&\ve{\hat{x}^{i-1}-x^{i-1}}\le \fc{(i-1)\epsilon' \ve{x^{i-1}}}d<\ve{x^{i-1}} \\
		\nonumber
	&\le (\eps'\mu_i\icu/3d)\|J^{i,j}_{x^i}\|\|A_i\|_F \|\phi(x^{i-1})\|&& \mbox{Activation Contraction} \\
		\nonumber
	&\le (\eps'\icu/3d)\|J^{i,j}_{x^i}\|\|A_i \phi(x^{i-1})\|&& \mbox{Layer Cushion} \\
		\nonumber
	& = (\eps'\icu/3d)\|J^{i,j}_{x^i}\|\|x^i\|&& x^i = A_i\phi(x^{i-1})\\
	\label{e:induct-error1}
	& \le (\eps'/3d) \|x^j\|&& \mbox{Interlayer Cushion} 
	\end{align}
	The second term of~\eqref{e:induct-error} can be bounded as:
	\begin{align}
	\nonumber
	& \|M^{i,j}(\hat{A}_i \phi(\hat{x}^{i-1})) - M^{i,j}(A_i \phi(\hat{x}^{i-1})) - J^{i,j}_{x^i}(\Delta_i\phi( \hat{x}^{i-1}))\| \\
	\nonumber
	& = \|(M^{i,j}-J^{i,j}_{x^i})(\hat{A}_i \phi(\hat{x}^{i-1})) - (M^{i,j}-J^{i,j}_{x^i})(A_i \phi(\hat{x}^{i-1}))\| \\
	\label{e:induct-error2}
	& \leq \|(M^{i,j}-J^{i,j}_{x^i})(\hat{A}_i \phi(\hat{x}^{i-1}))\| + \|(M^{i,j}-J^{i,j}_{x^i})(A_i \phi(\hat{x}^{i-1}))\|.
	\end{align}
Both terms of~\eqref{e:induct-error2} can be bounded using the interlayer smoothness condition. For the second term of~\eqref{e:induct-error2}, notice that $A_i\phi(\hat{x}^{i-1}) = \hat{x}^i_{i-1}$. Thus by the induction hypothesis, we know 
\begin{align}\|A_i\phi(\hat{x}^{i-1}) - x^i\|
= \|\hat{x}^i_{i-1} - x^i\| \le (i-1)\eps'\|x^i\|/d \le \eps'\|x^i\|.
\label{e:induct-xii1}
\end{align}
Now, by interlayer smoothness, 
\begin{align}
\nonumber
\n{(M^{i,j}-J^{i,j}_{x^i})(A_i \phi(\hat{x}^{i-1}))}
&=
\n{(M^{i,j}-J^{i,j}_{x^i})(x^i+(A_i\phi (\hat{x}^{i-1})-x^i))}\\
\nonumber
&\le \fc{\n{A_i\phi (\hat{x}^{i-1})-x^i}\n{x^j}}{\rho \n{x^i}}\\
\label{e:induct-error22}
&\stackrel{(*)}\le \fc{\epsilon'\n{x^i}\n{x^j}}{3d\n{x^i}} =\fc{\eps'\n{x^j}}{3d}
\end{align}
where in (*) we use~\eqref{e:induct-xii1} and the assumption $\rho\ge 3d.$
For the first term of~\eqref{e:induct-error2}, we know $\hat{A}_i\phi(\hat{x}^{i-1}) = \hat{x}^i_{i-1} + \Delta_i\phi(\hat{x}^{i-1})$. Therefore by the induction hypothesis and~\eqref{e:induct-error1} for $i=j$,
\begin{align*}\|\hat{A}_i\phi(\hat{x}^{i-1}) - x^i\| \le \|\hat{x}^i_{i-1} - x^i\| + \|\Delta_i\phi(\hat{x}^{i-1})\| \le (i-1)\eps'\n{x^i}/d+\eps'\n{x^i}/3d \le \eps'\n{x^i},
\end{align*}
so again we have 
\begin{align}\label{e:induct-error21}
\|(M^{i,j}-J^{i,j}_{x^i})(\hat{A}_i \phi(\hat{x}^{i-1}))\| \le (\eps'/3d)\|x^j\|.
\end{align}
Together,~\eqref{e:induct-error21} and~\eqref{e:induct-error22} show that~\eqref{e:induct-error2} is $\le \fc{2\eps'}{3d}\n{x^j}$. Together with~\eqref{e:induct-error1} we obtain from~\ref{e:induct-error} that $\|\hat{x}^j_i - \hat{x}^j_{i-1}\|\le \fc{\eps'}{d} \n{x^j}$, and hence that $\n{\hat x_i^j - x^j}\le \fc{i\eps' \ve{x^j}}{d}$, completing the induction step.

Conditioning on the success of interlayer smoothness and inequality~(\ref{eq:noise_layer}), we've shown,
	$$
	\|\hat{x}^j_i - x^j\|\le (i/d)\eps'\|x^j\|,
	$$
	for any $i\leq j\leq d$. Recall that with probability at least $1/2-\delta$, interlayer smoothness holds and inequality~(\ref{eq:noise_layer}) holds for every $2\leq i\leq d.$ Thus, let $\epsilon'=O\pa{\sqrt{\frac{p c^2d^2\log(md h_d/\delta)}{h_{\min}\min\limits_{2\leq i\leq d}\mu_i^2\icu^2 }}}$, we know with probability at least $1/2-\delta,$
	 $$\|f_\theta(x) - f_{\theta_1}(x)\| =\|x^d - \hat x^d_d\| \le \eps'\|f_{\theta}(x)\|.$$
	 
\paragraph{Bounding $\|f_\theta(x) - f_{\theta_t}(x)\|$ for any fixed $t$:}

 The proof for a fixed network on the path is almost the same as the proof for the end point. Instead of considering $\hat{x}_i^j$, now we consider $\hat{x}^j_i(t)$, which is the vector before activation at layer $j$ if the weights $A_2,\dots,A_i$ are replaced by $A_2+t(\hat{A}_2-A_2),\dots,A_i+t(\hat{A}_i-A_i)$. We can still use Lemma~\ref{lem:dropout_layer} to bound the noise produced by replacing the weight matrix at a single layer because 
$$\n{U(A_i+t(\hat{A}_i-A_i)-A_i)x}=t\n{U(\hat{A}_i-A_i)x}\leq \n{U(\hat{A}_i-A_i)x}.$$
Thus, we can still use the above induction proof to show that for any fixed $0\leq t\leq 1,$ let $\epsilon'=O\pa{\sqrt{\frac{p c^2d^2\log(md h_d/\delta)}{h_{\min}\min\limits_{2\leq i\leq d}\mu_i^2\icu^2 }}}$, with probability at least $1/2-\delta,$
	 $$\|f_\theta(x) - f_{\theta_t}(x)\| \le \eps'\|f_{\theta}(x)\|.$$

\paragraph{Bounding $\|f_\theta(x) - f_{\theta_t}(x)\|$ for every $t$:} Finally, we show that $\|f_\theta(x) - f_{\theta_t}(x)\|$ is bounded for every point on the path via an $\eps'$-net argument. Similar to previous steps, letting $\epsilon'=O\pa{\sqrt{\frac{p c^2d^2\max_{x\in S}(\ns{f_{\theta}(x)})\log(md h_d/\delta)}{h_{\min}\min\limits_{2\leq i\leq d}(\mu_i^2\icu^2)}}}$, we know that with probability at least $1/2-\delta,$
	 $$\|f_\theta(x) - f_{\theta_1}(x)\| \le \epsilon'/2.$$

Next, we show that on the path, the network output is smooth in terms of the parameters. According to Algorithm~\ref{alg:dropout}, we know for any $2\leq i\leq d$, we have $\n{\hat{A}_i}\leq 4\n{A_i}$, so $\n{\hat A_i-A_i}\le 5\ve{A_i}$. 
For any $2\le i\le d$, let $A_{i,t}=A_i + t(\hat A_i - A_i)$. Note $\n{A_{i,t}}\leq (1-t)\n{\hat{A}_i}+t\n{A_i} \le 4\n{A_i}$.
For any $t,t'$ and any $2\leq i\leq d$, let $\theta_{t,t'}^i$ be $\theta_{t}$ with the weight matrix at every layer $2\leq j\leq i$ replaced by $(A^j+t'(\hat{A}^j-A^j))$. For convenience, we also denote $\theta_t$ as $\theta_{t,t'}^1.$
Given $\tau<1/2,$ for any $\tau \leq t\leq 1-\tau$ and for any $-\tau \leq \kappa \leq \tau,$ we can bound $\n{f_{\theta_{t+\kappa}}(x)-f_{\theta_t}(x)}$ as follows:
\begin{align*}
\n{f_{\theta_{t+\kappa}}(x)-f_{\theta_t}(x)} &\leq \sum_{2\leq i\leq d}\n{f_{\theta^{i}_{t,t+\kappa}}(x)-f_{\theta_{t,t+\kappa}^{i-1}}(x)}
\end{align*}
The output of layer $i-1$ is the same for the two networks, of norm $\le \n{x}\prodo j{i-1} \n{A_{j,t+\kappa}}$. Hence the output of layer $i$ differs by at most $\kappa \n{x}\n{\hat A_i -A_i}\prodo j{i-1}\n{A_{j,t+\kappa}}$ and the output differs by $\kappa \n{x}\n{\hat A_i-A_i}\prodo j{i-1} \n{A_{j,t+\kappa}}\prod_{j=i+1}^d \n{A_{j,t}} \le 5^d\kappa\n{x} \prodo jd \n{A_j}$. Hence
\begin{align*}
\n{f_{\theta_{t+\alpha}}(x)-f_{\theta_t}(x)} 
&\leq \sum_{2\leq i\leq d}5^d\n{x}\kappa\prod_{1\leq j\leq d}\n{A_j}\\
&\leq 5^dd\kappa\n{x}\prod_{1\leq j\leq d}\n{A_i}.
\end{align*}

Thus, given $\tau\leq \frac{\epsilon'}{2\cdot 5^dd\max\limits_{x\in S}\n{x}\prod_{1\leq j\leq d}\n{A_j}},$ we know for any $\tau \leq t\leq 1-\tau$ and for any $-\tau \leq \alpha \leq \tau,$ 
\begin{equation}
\n{f_{\theta_{t+\alpha}}(x)-f_{\theta}(x)}\leq \epsilon'/2.\label{eq:epsnet}
\end{equation}

There exists a set $Q=\{\theta_t\}$ with size $O(1/\tau)$ such that for any network on the path, the distance to the closest network in $Q$ is no more than $\tau.$ If we can prove for any $\theta_t\in Q$, $\|f_\theta(x) - f_{\theta_t}(x)\| \le \epsilon'/2$, we immediately know for any network $\theta_{t'}$ on the path $\|f_\theta(x) - f_{\theta_{t'}}(x)\| \le \epsilon'$ by inequality~(\ref{eq:epsnet}).

By a union bound over $Q$, letting $\epsilon'=O\pa{\sqrt{\frac{p c^2d^2\max_{x\in S}(\ns{f_{\theta}(x)})\log\pf{md h_d}{\delta\tau}}{h_{\min}\min\limits_{2\leq i\leq d}(\mu_i^2\icu^2) }}}$, we know with probability at least $1/2-\delta,$
	 $$\|f_\theta(x) - f_{\theta_t}(x)\| \le \epsilon'/2,$$
	 for any $\theta_t\in Q$. 
	 
Setting $\delta=1/4$, we know there exists
$$\epsilon'=O\pa{\sqrt{\frac{pc^2d^3\max_{x\in S}(\ns{f_{\theta}(x)}) \log \pa{\frac{md h_d\max\limits_{x\in S}\n{x}\prod_{1\leq j\leq d}\n{A_j}}{ \epsilon'}}}{h_{\min}\min\limits_{2\leq i\leq d}(\mu_i^2\icu^2) }}}$$ such that with probability at least $1/4$, 
$$\|f_\theta(x) - f_{\theta_t}(x)\| \le \epsilon'$$
for any $x\in S$ and any $0\leq t\leq 1.$ Since the loss function is $\beta$-Lipschitz, we further know that for any $0\leq t\leq 1$:
	$$
	L(f_{\theta_t}) \le L(f_{\theta})+\beta\eps'= L(f_{\theta})+ \tdo(\sqrt{p}\epsilon).
	$$
\end{proofof}

Now, we are ready to prove the main theorem.
\noisestableconnect*
\begin{proofof}{Theorem~\ref{thm:noise_stable_connect}}
Setting dropout probability $p= 3/4$, by Lemma~\ref{lem:dropout_layer} and Lemma~\ref{lem:dropout_crosslayer}, if $h_{\min}=\tdomega\pa{1}$, we know there exist $\theta^A_1$ and $\theta^B_1$ such that
\begin{enumerate}
    \item in both networks, each weight matrix from layer $2$ to layer $d$ has at least half of columns as zero vectors;
    \item $L(f_{\theta^A_t})\leq L(f_{\theta^A})+\tdo(\epsilon)$ and  $L(f_{\theta^B_t})\leq L(f_{\theta^B})+\tdo(\epsilon)$, for any $0\leq t\leq 1$, where $\theta^A_t=\theta^A+t(\theta^A_1-\theta^A)$ and $\theta^B_t=\theta^B+t(\theta^B_1-\theta^B)$.
\end{enumerate}

Since the dropout fraction in both $\theta^A_1$ and $\theta^B_1$ is at least half, we can connect $\theta^A_1$ and $\theta^B_1$ as we did in Lemma~\ref{lem:drop-connect}, while ensuring the loss doesn't exceed 
$\max \{L(f_{\theta^A}),L(f_{\theta^B})\}+\tdo(\epsilon)$. Connecting $\theta^A$ to $\theta^A_1$ and connecting $\theta_B$ to $\theta^B_1$ each take one line segment. By the construction in Lemma~\ref{lem:drop-connect}, connecting two dropped-out networks $\theta^A_1$ and $\theta^B_1$ takes $8$ line segments. Thus, overall the path between $\theta^A$ and $\theta^B$ contains $10$ line segments.
\end{proofof}

Next, we show that if there exists a ``narrow'' neural network achiving small loss, we can get a lower energy barrier using a smaller dropout probability. 
\teacherstudent*

\begin{proofof}{Theorem~\ref{cor:teacher_student}}
Since $h_{\min}\cdot \max_{1\leq i\leq d-1}(h_i^*/h_i)\geq h_{\min}^* =\tdomega(1),$ we have $h_{min}=\tdomega\pf{1}{\max_{1\leq i\leq d-1}(h_i^*/h_i)}.$ By Lemma~\ref{lem:dropout_layer} and Lemma~\ref{lem:dropout_crosslayer}, there exist $\theta^A_1$ and $\theta^B_1$ such that
\begin{enumerate}
    \item in both networks, each weight matrix from layer $2$ to layer $d$ has at least $h_i^*$ columns set to zero;
    \item $L(f_{\theta^A_t})\leq L(f_{\theta^A})+\tdo(\sqrt{p}\epsilon)$ and  $L(f_{\theta^B_t})\leq L(f_{\theta^B})+\tdo(\sqrt{p}\epsilon)$, for any $0\leq t\leq 1$, where $\theta^A_t=\theta^A+t(\theta^A_1-\theta^A)$ and $\theta^B_t=\theta^B+t(\theta^B_1-\theta^B)$.
\end{enumerate}

From the fact that at least $h_i^*$ units in layer $i$ of both $\theta^A_1$ and $\theta^B_1$ have been set to zero for $1 \leq i < d$---meaning that the corresponding rows of $A_i$ and columns of $A_{i+1}$ are zero---it follows from Lemma~\ref{lem:drop-connect} that we can connect $\theta^A_1$ to an arbitrary permutation of $\theta^*$ using $8$ segments while keeping the loss on the path no more than $\max\{L(f_{\theta^A_1}),L(f_{\theta^*})\}.$ By choosing this permutation so that the non-zero units of $\theta^*$ do not intersect with those of $\theta^B_1$, we can then connect $\theta^*$ to $\theta^B_1$ using just $3$ segments as done in the first step of our path construction in Lemma~\ref{lem:drop-connect} seeing as there is no need to permute $\theta^*$ a second time. Combining these paths together with the paths that interpolate between the original parameters $\theta^A$ and $\theta^B$ and their dropout versions $\theta^A_1$ and $\theta^B_1$, we obtain a path in parameter space $\pi : [0,1] \to \Theta$ between $\theta^A$ and $\theta^B$ with $13$ line segments such that $L(f_{\pi(t)}) \leq \max \{L(f_{\theta^A})+\tdo(\sqrt{p}\epsilon),L(f_{\theta^B})+\tdo(\sqrt{p}\epsilon), L(f_{\theta^*})\}$ for $0 \leq t \leq 1$.
\end{proofof}

%% file: counterexample.tex
\section{Proofs for disconnected modes in two-layer nets}
\label{sec:example}
\begin{proofof}{Theorem~\ref{thm:counterexample}}
Define our loss over parameter space such that $L(f_\theta) = \frac{1}{n}\sum l(y_i,f_\theta(\xx_i))$, where $\xx_i \in \RR^{h+2}$ is our $i^{\text{th}}$ data sample, $y_i \in \RR$ the associated label, and $f_\theta(\xx_i) = \ww^T \phi ( A \xx_i )$ for $\theta = (\ww,A) \in \R^{(h+2) \times h} \times \R^{h}$. We can represent the data samples as rows in a matrix $X \in \RR^{n \times (h+2)}$---with $\mathbf{f}_i$ denoting the $i^{\text{th}}$ ``feature'' (i.e. column) of $X$---and the labels as elements of $\mathbf y \in \RR^{n \times 1}$, as illustrated in Figure~\ref{fig:counterexample}.

Choose $k,l,m,n$ such that $k < l < m < n$ where $k > h$, $l - k > h$, $m-l > 2$ and $n-m > h$.  

When $i \leq l$, let 
\begin{align*}
    x_{i,j} &= \begin{cases}
        i,&j=1\\
        i-1,&j=2\\
        1,&i\equiv j\pmod h\\
        -1,&i\nequiv j \pmod h,\, i\le l\\
        0,&i\nequiv j \pmod h,\, k<i\le l.
    \end{cases}
\end{align*}

When $l<i\le m$, let 
\begin{align*}
    x_{i,j} &= \begin{cases}
        -1, &j\le 2,\, i\equiv j\pmod 2\\
        0, &j\le 2,\, i\nequiv j \pmod 2\\
        0, &j>2,\, l<i\le m.
    \end{cases}
\end{align*}

When $i>m$, let
\begin{align*}
    x_{i,j} &=\begin{cases}
        0, &j\le 2\\
        -1, &j>2,\, i\equiv j\pmod h\\
        0, &j>2,\, i\nequiv j\pmod h.
    \end{cases}
\end{align*}

Finally, let $y_i = 1$ when $i \leq l$ and $0$ otherwise. 

\begin{figure}[H]
$$%
X = \begin{blockarray}{ccccccc}
 & \mathbf{f}_1 & \mathbf{f}_2 & \mathbf{f}_3 & \mathbf{f}_4 & \dots & \mathbf{f}_{h+2} \\
\begin{block}{c[cc|cccc]}
\xx_1 & 1 & 0 & 1 & -1 & \cdots & -1 \bigstrut[t] \\
\xx_2 & 2 & 1 & -1 & 1 & \ddots & \vdots \\
\vdots & \vdots & \vdots & \vdots & \ddots & \ddots & -1 \\
\vdots & \vdots & \vdots & -1 & \cdots & -1 & 1 \\
\xx_k & \vdots & \vdots & \vdots & \vdots & \vdots & \vdots \\
\cmidrule{4-7}
\vdots & \vdots & \vdots & \BAmulticolumn{4}{c}{\multirow{2}{*}{$I_h$}}\\
\xx_l & l & l-1 & \vdots & \ddots & \ddots & \ddots \\
\cmidrule{2-7}
 \vdots & \BAmulticolumn{2}{c|}{\multirow{2}{*}{$-I_2$}} & \BAmulticolumn{4}{c}{\multirow{2}{*}{$0$}}\\
 \vdots & & & & & & \\
 \xx_m & \vdots & \ddots & \vdots & \vdots & \vdots & \vdots \\
 \cmidrule{2-7}
 \vdots & \BAmulticolumn{2}{c|}{\multirow{2}{*}{$0$}} & \BAmulticolumn{4}{c}{\multirow{2}{*}{$-I_{h}$}} \\
 \vdots & & & & & & \\
\xx_n & \vdots & \vdots & \vdots & \ddots & \ddots & \ddots \bigstrut[b]\\
\end{block}
\end{blockarray}\vspace*{-1.25\baselineskip}
\quad \mathbf y = \begin{blockarray}{cc}
\\
\begin{block}{c[c]}
y_1 & 1 \bigstrut[t] \\
\vdots & \vdots \\
\vdots & \vdots \\
\vdots & \vdots \\
\vdots & \vdots \\
\vdots & \vdots \\
y_l & \vdots \\
\cmidrule{2-2}
\vdots & 0 \\
\vdots & \vdots \\
\vdots & \vdots \\
\vdots & \vdots \\
\vdots & \vdots \\
y_n & \vdots \bigstrut[b] \\
\end{block}
\end{blockarray}\vspace*{-1.25\baselineskip}
$$
\vspace{0.5cm}
\caption{Our dataset.}
\label{fig:counterexample}
\end{figure}

From the fact that $\phi(\mathbf{f}_1) - \phi(\mathbf{f}_2) = \sum_{j=3}^{h+2} \phi(\mathbf{f}_j) = \mathbf y$ 
it follows that there exist networks with both two active hidden units and $h$ active hidden units that achieve minimal loss, with the former corresponding to the ground truth teacher network which generated our dataset. 

Note in particular that the output layer weight corresponding to $\phi(\mathbf{f}_2)$ in the network with two active hidden units is negative, whereas in the network with $h$ active hidden units the output layer weights are all positive. Thus, any path between the two networks must pass through a point in parameter space where at least one output layer weight is zero while the other $h-1$ are positive. However, as shown in Lemma \ref{lem:emp-counterexample}, there does not exist such a point in parameter space that achieves minimal loss. It follows that there exists a barrier in the loss landscape separating the original networks, both of which are global minima. Moreover, by adjusting $k$, $l$, and $m$ we can somewhat arbitrarily raise or lower this barrier.
\end{proofof}

\begin{lemma}\label{lem:emp-counterexample}
There does not exist a set of $h-1$ positive weights $w_i$ and vectors $\hh_i \in \spn X$ such that $\sum_{i=1}^{h-1} w_i \phi(\hh_i) = \mathbf y$.
\end{lemma}
\begin{proof}
We can think of each $\hh_i$ as the output a particular hidden unit over all $n$ samples in our dataset and $w_i$ as the output layer weight associated to this hidden unit. We then have $\hh_i = \sum a_{i,j} \mathbf{f}_j$, where the coefficients $a_{i,j}$ are elements of $A$.

First, if there did exist $w_i$ and $\hh_i$ such that $\sum_{i=1}^{h-1} w_i \phi(\hh_i) = \mathbf y$, then it must be the case for all $i$ that $\hh_i = \sum a_{i,j} \mathbf{f}_j$ where $a_{i,j} \ge  0$ for all $j$. Otherwise, there would be non-zero elements in some $\hh_i$ between indexes $l+1$ and $n$ that would be impossible to eliminate in $\sum_{i=1}^{h-1} w_i \phi(\hh_i)$ given that $w_i > 0$ for all $i$.

Second, any linear combination of $\mathbf{f}_1$ and $\mathbf{f}_2$ with positive coefficients would result in a vector whose first $l$ elements are positive and increasing. In contrast, the first $l$ elements of $Y$ are constant. And so from the fact that there does not exist $a_{i,j} > 0$ such that the first $l$ elements of $\sum a_{i,j} \mathbf{f}_j$ are decreasing---in particular because the first $k$ elements and next $l-k$ elements of $\sum_{j=3}^{h+2} a_{ij}x_j$ are periodic with length $h$---it follows that $a_{i,1},a_{i,2} = 0$ for all $\hh_i$.


Thus, we need only consider linear combinations of $\mathbf{f}_3$ through $\mathbf{f}_{h+2}$ with positive coefficients as candidates for $\hh_i$. To this end, note that if a particular $\mathbf{f}_j$ has zero coefficient in all of $\hh_1$ through $\hh_{h-1}$, then $\sum_{i=1}^{h-1} w_i \phi(\hh_i)$ will have zeros in every index congruent to $j \mod h$ and therefore cannot equal $\mathbf y$. Hence by the pigeonhole principle, in order to have $\sum_{i=1}^{h-1} w_i \phi(\hh_i) = \mathbf y$ there must be some $i$ such that $\hh_i = \sum_{j=3}^{h+2} a_{i,j} \mathbf{f}_j$ with at least two coefficients being non-zero. However, in any linear combination $\sum_{j=3}^{h+2} a_{i,j} \mathbf{f}_j$ where $a_{i,j},a_{i,j'} > 0$ for at least two distinct $j,j'$, the elements in indexes $k+1$ to $l$ will be greater than the elements in indexes $1$ to $k$ that are congruent to $j \mod h$ and $j' \mod h$. In contrast, the first $l$ elements of $\mathbf y$ are constant. Hence, similar to the case of $\mathbf{f}_1$ and $\mathbf{f}_2$, there cannot exist $\hh_i = \sum_{j=3}^{h+2} a_{i,j} \mathbf{f}_j$ and positive coefficients $w_i$ such that $\sum_{i=1}^{h-1} w_i \phi(\hh_i) = Y$.
\end{proof}

%% file: appendix_experiments.tex
\section{Experimental details and further results}
\label{sec:extra-experiments}

\subsection{Experimental details and hyperparameters}\label{appdix:hyperparameters}

For all experiments on MNIST, we used a convolutional architecture consisting of $3$ convolutional layers followed by a fully-connected output layer. Each convolutional layer consisted of $32$ $3 \times 3$ filters and used sufficient padding so as to keep the layer's output the same shape as its input. All networks were trained on an NVIDIA Tesla K20c GPU for $5000$ iterations with a batch size of $64$ using stochastic gradient descent with an initial learning rate of $0.1$ and a decay rate of $1\text{E}^{-6}$. No significant hyperparameter tuning was applied. Images were normalized.

For the left and right plots in Figure~\ref{fig:mnist}, we report results averaged over $5$ random trials and error bars corresponding to the standard deviation over these trials. For the center plot we simply computed the loss and accuracy over a linear path between a particular convolutional net and a single dropout version of itself. Specific to Figure~\ref{fig:mnist}, in applying dropout with probability $p$ we randomly sample a subset of $\lfloor 32 (1-p) \rfloor$ units and rescale these units by $1/(1-p)$ while setting the remaining units to zero. In the left plot, each trial consisted of sampling $20$ such dropout networks and reporting the performance of the network achieving the lowest loss. Losses and accuracies in all plots were computed on a random batch of $4096$ training images.

On CIFAR-10, we trained VGG-11 networks on an NVIDIA Titan X GPU for 300 epochs with SGD with a batch size of 128, with weight
decay 5e-4, momentum 0.9, and an initial learning rate of 0.05 which is decayed by factor of 2 every 30 epochs. We used channel-wise dropout at all convolutional layers. The dropout rates are $p=0.25$ at the first three layers and are $p=0.5$ at the others. Ordinary dropout with $p=0.5$ is used at every fully-connected layers except for the last one (the softmax layer).

\pagebreak

\subsection{Straight interpolation between two models}
\label{sec:direct_interpolation}
As demonstrated in Figure~\ref{fig:direct_interpolation}, a straight line interpolation between two noise stable model may incur \emph{large} losses and \emph{poor} accuracies. The models are the same as used in Figure~\ref{fig:properties}.
\hspace{-4cm}
\begin{figure}[H]
	\centering
	\includegraphics[width=7cm]{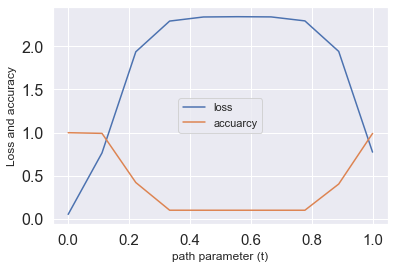}
	\caption{Loss and accuracy from directly interpolating between two noise stable models.}
	\label{fig:direct_interpolation}
\end{figure}
\subsection{Verification of noise stability conditions}\label{appdix:verification}


\subsubsection{Layer cushion}

\begin{figure}[H]
    \centering
    \def\mydata{
    1,2,3,4,5,6,7,8}
    \foreach \x in \mydata
    {
    \begin{subfigure}[t]{0.3\textwidth}
        \centering
        \includegraphics[height=2.9cm, width=3.5cm]{figures/cushion_layer_\x.pdf}
    \end{subfigure}
    }
\end{figure}

\subsubsection{Interlayer cushion}\label{appendix:interlayer-cushion}

\begin{figure}[H]
    \centering
    \def\mydata{
    1,2,3,4,5,6,7,8}
    \foreach \x in \mydata
    {
    \begin{subfigure}[t]{0.3\textwidth}
        \centering
        \includegraphics[height=2.9cm, width=3.5cm]{figures/interlayer_cushion_layer_\x.pdf}
    \end{subfigure}
    }
\end{figure}

\subsubsection{Activation contraction}

\begin{figure}[H]
    \centering
    \def\mydata{
    1,2,3,4,5,6,7,8}
    \foreach \x in \mydata
    {
    \begin{subfigure}[t]{0.3\textwidth}
        \centering
        \includegraphics[height=2.9cm, width=3.5cm]{figures/contraction_layer_\x.pdf}
    \end{subfigure}
    }
\end{figure}

\subsubsection{Interlayer smoothness}

\begin{figure}[H]
    \centering
    \def\mydata{
    2,3,4,5,6,7,8}
    \foreach \x in \mydata
    {
    \begin{subfigure}[t]{0.3\textwidth}
        \centering
        \includegraphics[height=2.9cm, width=3.5cm]{figures/smoothness_layer_\x.pdf}
    \end{subfigure}
    }
\end{figure}

%% file: tools.tex
\section{Tools}
We use matrix concentration bounds to bound the noise produced by dropping out one single layer (Lemma~\ref{lem:dropout_layer}).

\begin{lemma}[Matrix Bernstein; Theorem 1.6 in \citep{tropp2012user}]\label{lm:bernstein_original}
Consider a finite sequence $\{Z_k\}$ of independent, random matrices with dimension $d_1\times d_2$. Assume that each random matrix satisfies
$$\E[Z_k]=0\ and \ \n{Z_k}\leq R \ almost \ surely. $$
Define 
$$\sigma^2:= \max\Big\{\big\lVert\sum_k \E[Z_k Z_k^*]\big\rVert, \big\lVert \sum_k \E[Z_k^* Z_k]\big\rVert \Big\}.$$
Then, for all $t\geq 0$, 
$$\Pr\Big\{\big\lVert\sum_k Z_k\big\rVert\geq t \Big\}\leq (d_1+d_2) \exp\Big(\frac{-t^2/2}{\sigma^2+Rt/3}\Big).$$
\end{lemma} 

As a corollary, we have:
\begin{lemma}[Bernstein Inequality: Vector Case]\label{lm:vectorBernstein}
Consider a finite sequence $\{v_k\}$ of independent, random vectors with dimension $d$. Assume that each random vector satisfies
$$\n{v_k-\E[v_k]}\leq R\ almost\ surely.$$
Define
$$\sigma^2 := \sum_{k}\E\big[\ns{v_k-\E[v_k]}\big].$$
Then, for all $t\geq 0$,
$$\Pr\Big\{\n{\sum_{k}(v_k-\E[v_k])}\geq t\Big\}\leq (d+1)\cdot \exp\Big(\frac{-t^2/2}{\sigma^2+Rt/3}\Big).$$
\end{lemma}